\documentclass{article}
 \usepackage[preprint]{log_2026}			

\usepackage{booktabs}						
\usepackage{multirow}						
\usepackage{graphicx}						
\usepackage{amsfonts}						
\usepackage{amssymb}
\usepackage{amsthm}							
\usepackage{mathtools}
\usepackage{float}

\usepackage{natbib}

\usepackage{tikz}
\usetikzlibrary{positioning, arrows.meta, calc, fit, backgrounds, shapes.geometric, decorations.pathreplacing, decorations.pathmorphing}

\definecolor{palBlack}  {HTML}{212529}    
\definecolor{palJet}    {HTML}{343a40}    
\definecolor{palCream}  {HTML}{e9ecef}    
\definecolor{palKhaki}  {HTML}{ced4da}    
\definecolor{palStone}  {HTML}{495057}    

\definecolor{ink}    {HTML}{212529}
\definecolor{mute}   {HTML}{6c757d}
\definecolor{rule}   {HTML}{dee2e6}
\definecolor{accent1}{HTML}{495057}
\definecolor{accent2}{HTML}{343a40}
\definecolor{tint1}  {HTML}{e9ecef}
\definecolor{tint2}  {HTML}{f8f9fa}

\newtheorem{proposition}{Proposition}
\newtheorem{remark}{Remark}

\title[Learning a Per-Edge Tsallis Index]{When Should Graph Attention Be Sparse? \\Learning a Per-Edge Tsallis Index}

\author[da Costa et al.]{%
  Kleyton da Costa\thanks{Corresponding author.} \\
  University College London \& Holistic AI\\
  London, UK \\
  \email{kleyton.vsc@gmail.com} \\
  \And
  Bernardo Modenesi \\
  University of Utah\\
  Salt Lake City, UT, USA \\
  \email{bernardo.modenesi@utah.edu}
}

\begin{document}

\maketitle

\begin{abstract}
Graph attention normalizes neighborhood scores with softmax, the maximum-entropy choice under Shannon statistics.  But homophilic and heterophilic graphs want different attention shapes, and one fixed normalization cannot serve both.  We propose \textbf{LTGA} (\textbf{L}earnable \textbf{T}sallis \textbf{G}raph \textbf{A}ttention), a graph attention layer whose Tsallis entropic index $q$ is learned jointly with the weights, interpolating continuously between heavy-tailed ($q\!<\!1$), softmax ($q\!=\!1$) and compact-support ($q\!>\!1$) attention at four granularities from a global scalar to a per-edge index, under a bounded reparameterization that starts every model at the GAT baseline.  Across eight benchmarks at ten seeds, LTGA-Edge takes the best average rank ($2.75$), but the omnibus test does not reject ($p\!=\!0.199$) and learning $q$ does not beat searching it: a validation-tuned frozen grid reaches $61.4\%$, tuned $\alpha$-entmax $62.2\%$ and a capacity-matched $q\!\equiv\!1$ control $62.0\%$, against $61.7\%$ for LTGA-Edge.  What the learned index buys is one run instead of a grid, and an interpretable mechanism: where $q$ leaves $1$, it prunes $42\%$ of attention coefficients to exactly zero, and those edges are selectively the wrong ones, restoring them costs $7.1$ points, while random pruning at the same rate costs $13.0$ more. \\\vspace{0.5cm} Project page: \href{https://kleyt0n.github.io/ltga}{https://kleyt0n.github.io/ltga}
\end{abstract}

\section{Introduction}\label{sec:introduction}

Graph attention networks~(GAT)~\citep{velickovic2018graph} normalise pairwise relevance scores into a probability distribution over each node's neighbourhood with a softmax.  Softmax is the unique solution to the constrained \emph{Shannon} maximum-entropy problem~\citep{jaynes1957information}, so this design implicitly assumes that attention weights are most usefully spread according to extensive (Boltzmann--Gibbs) statistics.  Real-world graphs, however, routinely violate the assumptions that justify extensivity.  Citation networks are dominated by a few highly relevant neighbours, and the long tail of weakly related ones is best ignored entirely; heterophilic web graphs benefit from the opposite behaviour, where many dissimilar neighbours each contribute a small amount of evidence~\citep{pei2020geomgcn, zhu2020beyond, platonov2023critical}.  In both regimes, dense softmax attention is sub-optimal because it cannot represent exact zeros and forces a fixed tail behaviour.

Two extremes have already been explored.  \emph{Sparsemax}~\citep{martins2016softmax} replaces softmax with a Euclidean projection onto the simplex, producing exact zeros at the cost of a fixed sparsity level.  $\alpha$-\emph{entmax}~\citep{peters2019sparse, correia2019adaptively} interpolates between softmax and sparsemax with a tunable~$\alpha$, and \citet{correia2019adaptively} additionally learn a per-head $\alpha$ inside a Transformer.  Both live in the same \emph{Tsallis} maximum-entropy family~\citep{tsallis1988possible, tsallis2009introduction} that we adopt, in which the entropic index $q$ controls a smooth transition from heavy-tailed ($q<1$) through Shannon ($q=1$) to compact-support ($q>1$) behaviour.  They are not, however, frozen points of our map: $\alpha$-entmax fixes its threshold by the simplex constraint whereas the $q$-softmax shifts by the neighbourhood maximum and normalises afterwards, so the two coincide only up to a per-neighbourhood rescaling of the logits (Appendix~\ref{app:entmax_relation}).  Our claim is therefore about the graph setting, the two-sided family, and the Shannon-limit gradients --- not about being the first to learn a sparsity index.

In this work we make the entropic index a \emph{learnable scalar} of the network.  Replacing softmax with a $q$-softmax inside GAT yields what we call \textbf{LTGA}: a graph attention layer whose normalisation continuously adapts between dense and sparse regimes during training, with no per-task grid search.  Three technical contributions make end-to-end learning stable: (i)~a $\tanh$-bounded reparameterisation $q = 1 + \delta\,\tanh(\alpha)$ that pins initialisation to the Shannon baseline and prevents divergent dynamics; (ii)~a decoupled optimiser with warm-up scheduling that lets the network weights converge to a strong feature regime before the attention geometry starts changing; and (iii)~a quadratic Shannon-prior regulariser that acts as an Occam's razor by penalising departures from $q=1$ that do not pay for themselves in task loss.  At $q\!=\!1$ the layer reduces to GATv2 exactly, an equality we verify numerically against the reference implementation rather than only asserting (Appendix~\ref{app:protocol}).

Our main contributions are:

\begin{itemize}
    \item We introduce LTGA, which brings a learnable entropic index to graph neighbourhood attention.  Unlike the fixed-$\alpha$ and per-head-$\alpha$ entmax variants, the family is \emph{two-sided} ($q\!<\!1$ heavy-tailed as well as $q\!>\!1$ compact), and the index can be conditioned per edge;
    \item We give a numerically stable formulation of $q$-softmax with a Taylor branch around $q=1$ that keeps both value \emph{and gradient} continuous across the Shannon limit, and we contribute reparameterisation, decoupled optimisation, and Shannon-prior regularisation as a set of training-time stabilisers;
    \item Across eight node-classification benchmarks spanning the homophily spectrum $h\in[0.05,0.81]$, we compare LTGA against softmax attention, a frozen-$q$ grid tuned per dataset, tuned $\alpha$-entmax and sparsemax attention, capacity-matched edge-gated controls, and four heterophily-specific architectures.  LTGA-Edge takes the best average rank; we report the Friedman $p$-value together with per-dataset paired tests (Holm-corrected) and Cohen's $d$, and we state plainly where the omnibus test fails to reject and where GATv2 remains ahead;
    \item We show that, when $q$ leaves the Shannon baseline, it does so \emph{interpretably}: an analytic decomposition links $q\!>\!1$ to data-dependent neighbourhood pruning (compact support) and $q\!<\!1$ to heavy-tailed diffusion. We test the pruning interpretation directly: pruned edges are compared against a rate-matched random control, and restoring them at inference quantifies what pruning buys. We also probe what the per-edge gate conditions on --- degree, feature similarity, and score gap --- rather than leaving it a black box.
\end{itemize}

\section{Background and Related Work}\label{sec:related}

\paragraph{Attention-based GNNs.}
GAT~\citep{velickovic2018graph} replaces fixed aggregation with a learnable softmax over each node's neighbourhood; GATv2~\citep{brody2022how} fixes a static-attention failure by relocating the LeakyReLU, an orthogonal improvement that LTGA inherits.  Graph transformers~\citep{ying2021transformers, rampasek2022recipe} extend attention beyond the 1-hop neighbourhood but retain softmax normalisation.  Our contribution is orthogonal to the scoring function and applies to any of these models.

\paragraph{Sparse and learnable attention.}
Sparsemax~\citep{martins2016softmax} and the $\alpha$-entmax family~\citep{peters2019sparse, correia2019adaptively} replace softmax with sparsity-inducing alternatives, with $\alpha$-entmax recovering softmax at $\alpha\!=\!1$ and sparsemax at $\alpha\!=\!2$.  Fenchel--Young losses~\citep{blondel2020fast} unify these via convex duality.  Adjacent ideas appear in classification losses with fixed Tsallis temperatures~\citep{de2019extensive, zhang2018generalized, amid2019robust} and in adaptive shape parameters for robust regression~\citep{barron2019general}.  Most of these expose the regulariser as a frozen hyperparameter, though \citet{correia2019adaptively} do learn a per-head $\alpha$ in sequence attention; the closest prior work is therefore adaptively-sparse Transformer attention, and we compare against a graph-adapted version of it directly (Section~\ref{sec:baseline_results}).  Two differences remain: $\alpha$-entmax is defined only for $\alpha\!\geq\!1$, so it cannot express the heavy-tailed regime, and its threshold rule differs from the $q$-softmax's max-shift (Appendix~\ref{app:entmax_relation}), which changes both the induced support and the gradients.  Tsallis statistics also appear in non-attention settings (Tsallis-INF~\citep{abernethy2014tsallis}, Tsallis-OT~\citep{muzellec2017tsallis}), with full mathematical background~\citep{tsallis1988possible, naudts2002deformed, furuichi2004fundamental} reviewed in Appendix~\ref{app:related}.

\paragraph{Heterophilic GNNs and calibration.}
Heterophilic graphs~\citep{pei2020geomgcn} have driven a wave of architectural responses --- H$_2$GCN~\citep{zhu2020beyond}, GPR-GNN~\citep{chien2021adaptive}, FAGCN~\citep{bo2021beyond}, LINKX~\citep{lim2021large}, GloGNN~\citep{li2022finding}, Polynormer~\citep{deng2024polynormer} --- and a careful benchmark re-evaluation~\citep{platonov2023critical} that we adopt.  LTGA operates at a different level: it modifies only the attention \emph{normalisation}, so it composes with any \emph{attention} scoring function, but not with these architectures, none of which normalises attention at all (Section~\ref{sec:discussion}).  We compare against them directly rather than claiming complementarity.  We further connect to the calibration literature~\citep{guo2017calibration, mukhoti2020calibrating}: by learning $q$, LTGA implicitly regularises output confidence and improves calibration without post-hoc temperature scaling (Appendix~\ref{app:calibration}).

\section{Methodology}\label{sec:methodology}

This section develops LTGA.  We begin with the $q$-softmax algebra (Section~\ref{sec:tsallis_prelim}); fold it into a graph attention layer (Section~\ref{sec:qgat_layer}); make the entropic index learnable via a $\tanh$-bounded reparameterisation (Section~\ref{sec:learnable_q}); and describe the decoupled optimiser and Shannon-prior regulariser that make end-to-end training reliable (Section~\ref{sec:training_protocol}).  The maximum-entropy derivation of $q$-softmax (Appendix~\ref{app:maxent}) and the analytic homophily prediction that motivates the empirical part of the paper (Appendix~\ref{app:graph_considerations}) are deferred for space.

\subsection{\texorpdfstring{Tsallis preliminaries: $q$-exp, $q$-log, and $q$-softmax}%
{Tsallis preliminaries: q-exp, q-log, and q-softmax}}\label{sec:tsallis_prelim}

The Tsallis entropy~\citep{tsallis1988possible, tsallis2009introduction} of $\mathbf{p}\in\Delta^{C-1}$ is $S_q(\mathbf{p})=(q{-}1)^{-1}\bigl(1-\sum_i p_i^{q}\bigr)$ for $q\!\neq\!1$, recovering Shannon entropy at $q\!\to\!1$.  We adopt the $(q{-}1)$ parameterisation that aligns with the $\alpha$-entmax~\citep{peters2019sparse} and sparsemax~\citep{martins2016softmax} literature, defining the $q$-exponential as
\begin{equation}\label{eq:q_log_exp}
    \exp_q(x) = \bigl[1 + (q-1)\,x\bigr]_+^{1/(q-1)},
    \qquad [\,\cdot\,]_+ = \max(\cdot,0),
\end{equation}
which recovers $\exp(x)$ as $q\!\to\!1$ via L'H\^{o}pital's rule and gives compact-support behaviour for $q\!>\!1$ (zeros where $x\!\leq\!-1/(q{-}1)$) and heavy tails for $q\!<\!1$.
and define the $q$-softmax by substitution:
\begin{equation}\label{eq:q_softmax}
    \mathrm{softmax}_q(\mathbf{z})_i = \frac{\exp_q(z_i - c)}{\sum_{j} \exp_q(z_j - c)},
    \qquad c = \max_j z_j.
\end{equation}
The shift by $c$ is required because $\mathrm{softmax}_q$ is not shift-invariant for $q\!\neq\!1$.  Two structural facts (proved in Appendix~\ref{app:tsallis}) underpin everything that follows:

\begin{proposition}[Regime dichotomy]\label{prop:dichotomy}
(i)~$\lim_{q\to 1}\mathrm{softmax}_q(\mathbf{z}) = \mathrm{softmax}(\mathbf{z})$, recovering GAT exactly.
(ii)~For $q\!>\!1$, $\mathrm{softmax}_q(\mathbf{z})_i = 0$ exactly when $z_i\leq c - 1/(q-1)$.
\end{proposition}

The two parts give the regime dichotomy that motivates the rest of the paper: $q\!>\!1$ produces sparse attention with hard zeros; $q\!=\!1$ produces dense softmax attention; $q\!<\!1$ produces heavy-tailed attention that retains weight on low-scoring neighbours.

\subsection{The LTGA layer}\label{sec:qgat_layer}

Let $\mathcal{G} = (\mathcal{V}, \mathcal{E})$ be a graph with node features $\mathbf{H}\in\mathbb{R}^{|\mathcal{V}|\times F}$ and optional edge attributes $\mathbf{e}_{ij}\in\mathbb{R}^{F_e}$.  Following the parameterisation of \citet{brody2022how}'s GATv2, an LTGA layer projects each endpoint with two separate weight matrices $\mathbf{z}_i^{(\mathrm{src})} = \mathbf{W}_l\mathbf{h}_i$ and $\mathbf{z}_j^{(\mathrm{dst})} = \mathbf{W}_r\mathbf{h}_j$, optionally adds a projected edge feature $\mathbf{z}_{ij}^{(\mathrm{edge})}=\mathbf{W}_e\mathbf{e}_{ij}$, and computes the pairwise attention logit as
\begin{equation}\label{eq:gat_logit}
    e_{ij} = \mathbf{a}^{\!\top}\,\mathrm{LeakyReLU}\!\bigl(\mathbf{z}_i^{(\mathrm{src})} + \mathbf{z}_j^{(\mathrm{dst})} + \mathbf{z}_{ij}^{(\mathrm{edge})}\bigr).
\end{equation}
This is the GATv2 scoring function (the original GAT v1 of \citet{velickovic2018graph} relocates the LeakyReLU \emph{after} the dot product and is recovered by setting $\mathbf{W}_l = \mathbf{W}_r$). LTGA replaces the softmax normalisation with the $q$-softmax of Eq.~\eqref{eq:q_softmax}:
\begin{equation}\label{eq:q_attention}
    \alpha_{ij}^{(q)}
    = \frac{\exp_q(e_{ij} - c_i)}{\displaystyle\sum_{k \in \mathcal{N}(i)} \exp_q(e_{ik} - c_i)},
    \qquad c_i = \max_{k\in\mathcal{N}(i)} e_{ik},
\end{equation}
and aggregates with an optional residual:
\begin{equation}\label{eq:qgat_agg}
    \mathbf{h}'_i
    = \sigma\!\Biggl(\sum_{j \in \mathcal{N}(i)} \alpha_{ij}^{(q)}\,\mathbf{z}_j^{(\mathrm{dst})}\Biggr) + \mathbf{W}_{\mathrm{res}}\mathbf{h}_i,
\end{equation}
with $\mathbf{W}_{\mathrm{res}}\!=\!0$ when residuals are disabled.  Self-loops are added to $\mathcal{N}(i)$ before message passing, mirroring the default of PyG's \texttt{GATConv} and \texttt{GATv2Conv}.  The scoring function in Eq.~\eqref{eq:gat_logit} is unchanged from GATv2, so any improvement to GAT scoring carries over to LTGA verbatim.  A scalar-$q$ LTGA layer adds at most $M$ extra parameters (one per head); the per-edge variant (Section~\ref{sec:learnable_q}) adds a small $\sim\!10^2$-parameter MLP.

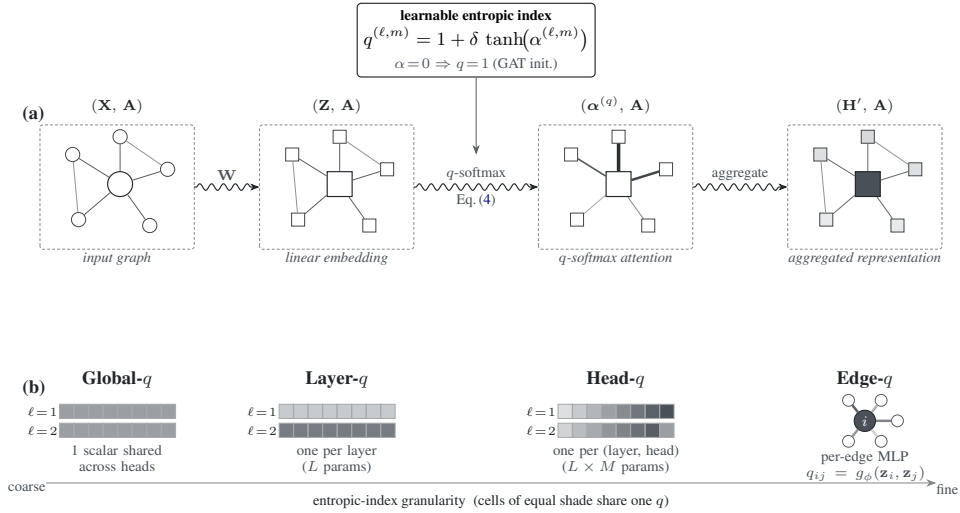
\begin{figure}[!t]
\centering
\resizebox{0.9\linewidth}{!}{%
\begin{tikzpicture}[
  x=1cm, y=1cm,
  font=\footnotesize,
  >={Stealth[length=1.6mm,inset=0.5mm]},
  every node/.append style={inner sep=0pt},
  inNode/.style    ={circle, draw=palBlack!85, line width=0.35pt, fill=white,
                     minimum size=2.4mm, inner sep=0pt},
  inHub/.style     ={circle, draw=palBlack, line width=0.55pt, fill=white,
                     minimum size=4.2mm, font=\scriptsize, inner sep=0pt},
  hNode/.style     ={rectangle, draw=palBlack!85, line width=0.35pt, fill=white,
                     minimum size=2.2mm, inner sep=0pt},
  hHub/.style      ={rectangle, draw=palBlack, line width=0.55pt, fill=white,
                     minimum size=4.2mm, font=\scriptsize, inner sep=0pt},
  stageBox/.style  ={draw=palBlack!55, line width=0.45pt, rounded corners=1.5pt,
                     dash pattern=on 1.4pt off 1.0pt, fill=white},
  stageLab/.style  ={font=\scriptsize, palBlack, anchor=south, align=center},
  arrowLab/.style  ={font=\scriptsize, palStone, fill=white, inner sep=1.5pt},
  snakeArr/.style  ={->, line width=0.55pt, draw=palBlack,
                     decorate, decoration={snake, amplitude=0.55mm,
                     segment length=2.0mm, post length=1.4mm, pre length=0.3mm}},
  qctl/.style      ={draw=palBlack, line width=0.55pt, rounded corners=2pt,
                     fill=white, align=center, inner sep=3pt},
  rpanel/.style    ={draw=palBlack!25, line width=0.4pt, rounded corners=2pt,
                     fill=white, minimum width=36mm, minimum height=28mm},
  rnbr/.style      ={circle, draw=palBlack!80, line width=0.3pt, fill=white,
                     minimum size=3.0mm, inner sep=0pt},
  rghost/.style    ={circle, draw=palBlack!22, line width=0.3pt, fill=white,
                     minimum size=3.0mm, inner sep=0pt},
  rhub/.style      ={circle, draw=palBlack, line width=0.55pt, fill=palStone,
                     text=white, minimum size=4.6mm, font=\scriptsize\bfseries,
                     inner sep=0pt},
]


\def\drawInputGraph{%
  \node[inNode] (a1) at (-0.75, 0.55) {};
  \node[inNode] (a2) at ( 0.05, 0.85) {};
  \node[inNode] (a3) at ( 0.85, 0.30) {};
  \node[inNode] (a4) at ( 0.55,-0.65) {};
  \node[inNode] (a5) at (-0.65,-0.55) {};
  \node[inHub]  (ah) at ( 0.05, 0.05) {};
  \draw[palBlack!75, line width=0.35pt]
      (ah)--(a1) (ah)--(a2) (ah)--(a3) (ah)--(a4) (ah)--(a5);
  \draw[palBlack!50, line width=0.30pt] (a1)--(a5) (a2)--(a3);
}
\def\drawHiddenGraph{%
  \node[hNode] (b1) at (-0.75, 0.55) {};
  \node[hNode] (b2) at ( 0.05, 0.85) {};
  \node[hNode] (b3) at ( 0.85, 0.30) {};
  \node[hNode] (b4) at ( 0.55,-0.65) {};
  \node[hNode] (b5) at (-0.65,-0.55) {};
  \node[hHub]  (bh) at ( 0.05, 0.05) {};
  \draw[palBlack!75, line width=0.35pt]
      (bh)--(b1) (bh)--(b2) (bh)--(b3) (bh)--(b4) (bh)--(b5);
  \draw[palBlack!50, line width=0.30pt] (b1)--(b5) (b2)--(b3);
}
\def\drawAttnGraph{%
  \node[hNode] (c1) at (-0.75, 0.55) {};
  \node[hNode] (c2) at ( 0.05, 0.85) {};
  \node[hNode] (c3) at ( 0.85, 0.30) {};
  \node[hNode] (c4) at ( 0.55,-0.65) {};
  \node[hNode] (c5) at (-0.65,-0.55) {};
  \node[hHub]  (ch) at ( 0.05, 0.05) {};
  \draw[palBlack, line width=1.7pt, opacity=0.95] (ch)--(c2);
  \draw[palBlack, line width=1.2pt, opacity=0.85] (ch)--(c3);
  \draw[palBlack, line width=0.8pt, opacity=0.75] (ch)--(c1);
  \draw[palBlack, line width=0.5pt, opacity=0.55] (ch)--(c4);
  \draw[palBlack, line width=0.3pt, opacity=0.40] (ch)--(c5);
}
\def\drawAggGraph{%
  \node[hNode, fill=palStone!18] (d1) at (-0.75, 0.55) {};
  \node[hNode, fill=palStone!22] (d2) at ( 0.05, 0.85) {};
  \node[hNode, fill=palStone!18] (d3) at ( 0.85, 0.30) {};
  \node[hNode, fill=palStone!12] (d4) at ( 0.55,-0.65) {};
  \node[hNode, fill=palStone!12] (d5) at (-0.65,-0.55) {};
  \node[hHub,  fill=palStone, text=white] (dh) at ( 0.05, 0.05) {};
  \draw[palBlack!70, line width=0.35pt]
      (dh)--(d1) (dh)--(d2) (dh)--(d3) (dh)--(d4) (dh)--(d5);
  \draw[palBlack!50, line width=0.30pt] (d1)--(d5) (d2)--(d3);
}

\foreach \k/\x in {1/0.0, 2/3.7, 3/8.4, 4/12.6}{
  \coordinate (S\k) at (\x, 0);
}

\def\stageHalfW{1.30}
\def\stageHalfH{1.05}
\foreach \k in {1,2,3,4}{
  \draw[stageBox]
    ($(S\k)+(-\stageHalfW,-\stageHalfH)$)
    rectangle ($(S\k)+(\stageHalfW,\stageHalfH)$);
}

\begin{scope}[shift={(S1)}] \drawInputGraph  \end{scope}
\begin{scope}[shift={(S2)}] \drawHiddenGraph \end{scope}
\begin{scope}[shift={(S3)}] \drawAttnGraph   \end{scope}
\begin{scope}[shift={(S4)}] \drawAggGraph    \end{scope}

\node[stageLab] at ($(S1)+(0,\stageHalfH+0.18)$) {$(\mathbf{X},\,\mathbf{A})$};
\node[stageLab] at ($(S2)+(0,\stageHalfH+0.18)$) {$(\mathbf{Z},\,\mathbf{A})$};
\node[stageLab] at ($(S3)+(0,\stageHalfH+0.18)$)
   {$(\boldsymbol{\alpha}^{(q)},\,\mathbf{A})$};
\node[stageLab] at ($(S4)+(0,\stageHalfH+0.18)$) {$(\mathbf{H}',\,\mathbf{A})$};

\node[font=\scriptsize\itshape, palStone, anchor=north]
   at ($(S1)+(0,-\stageHalfH-0.05)$) {input graph};
\node[font=\scriptsize\itshape, palStone, anchor=north]
   at ($(S2)+(0,-\stageHalfH-0.05)$) {linear embedding};
\node[font=\scriptsize\itshape, palStone, anchor=north]
   at ($(S3)+(0,-\stageHalfH-0.05)$) {$q$-softmax attention};
\node[font=\scriptsize\itshape, palStone, anchor=north]
   at ($(S4)+(0,-\stageHalfH-0.05)$) {aggregated representation};

\draw[snakeArr]
    ($(S1)+(\stageHalfW,0)$) --
    node[arrowLab, above=0.6mm] {$\mathbf{W}$}
    ($(S2)+(-\stageHalfW,0)$);

\draw[snakeArr]
    ($(S2)+(\stageHalfW,0)$) --
    node[arrowLab, above=0.6mm] {$q$-softmax}
    node[arrowLab, below=0.6mm] {Eq.\,(\ref{eq:q_attention})}
    ($(S3)+(-\stageHalfW,0)$);

\draw[snakeArr]
    ($(S3)+(\stageHalfW,0)$) --
    node[arrowLab, above=0.6mm] {aggregate}
    ($(S4)+(-\stageHalfW,0)$);

\node[qctl, anchor=south] (qctl)
   at ($(S2)!0.5!(S3) + (0,1.85)$)
   {\scriptsize\textbf{learnable entropic index}\\[1pt]
    $\displaystyle q^{(\ell,m)} = 1 + \delta\,\tanh\!\bigl(\alpha^{(\ell,m)}\bigr)$\\[1pt]
    \scriptsize\textcolor{palStone}{$\alpha\!=\!0 \Rightarrow q\!=\!1$ (GAT init.)}};
\draw[->, line width=0.45pt, draw=palBlack!75]
    (qctl.south) -- ($(S2)!0.5!(S3) + (0,0.55)$);

\node[font=\small\bfseries, palBlack, anchor=west]
    at (-1.6, \stageHalfH+0.18) {(a)};


\foreach \k/\x in {1/0.0, 2/3.7, 3/8.4, 4/12.6}{
  \coordinate (G\k) at (\x, -3.95);
}

\tikzset{
  gcell/.style     ={rectangle, draw=palBlack!40, line width=0.20pt,
                     minimum size=2.4mm, inner sep=0pt},
  gtitle/.style    ={font=\small\bfseries, palBlack, anchor=south, align=center},
  gsub/.style      ={font=\scriptsize, palStone, anchor=north, align=center,
                     text width=32mm},
  glayerLab/.style ={font=\tiny, palBlack, anchor=east},
}

\begin{scope}[shift={(G1)}]
  \foreach \r in {0,1}{
    \foreach \c in {0,1,2,3,4,5,6,7}{
      \pgfmathsetmacro{\xx}{-0.85 + 0.245*\c}
      \pgfmathsetmacro{\yy}{ 0.16 - 0.32*\r}
      \node[gcell, fill=palStone!55] at (\xx,\yy) {};
    }
  }
  \node[glayerLab] at (-1.00,  0.16) {$\ell\!=\!1$};
  \node[glayerLab] at (-1.00, -0.16) {$\ell\!=\!2$};
\end{scope}
\node[gtitle] at ($(G1)+(0, 0.55)$) {Global-$q$};
\node[gsub]   at ($(G1)+(0,-0.40)$) {1 scalar shared\\across heads};

\begin{scope}[shift={(G2)}]
  \foreach \c in {0,1,2,3,4,5,6,7}{
    \pgfmathsetmacro{\xx}{-0.85 + 0.245*\c}
    \node[gcell, fill=palStone!30] at (\xx,  0.16) {};
    \node[gcell, fill=palStone!75] at (\xx, -0.16) {};
  }
  \node[glayerLab] at (-1.00,  0.16) {$\ell\!=\!1$};
  \node[glayerLab] at (-1.00, -0.16) {$\ell\!=\!2$};
\end{scope}
\node[gtitle] at ($(G2)+(0, 0.55)$) {Layer-$q$};
\node[gsub]   at ($(G2)+(0,-0.40)$) {one per layer\\($L$ params)};

\begin{scope}[shift={(G3)}]
  \foreach \c [count=\ci from 0] in {18, 30, 42, 54, 66, 78, 90, 100}{
    \pgfmathsetmacro{\xx}{-0.85 + 0.245*\ci}
    \node[gcell, fill=palStone!\c] at (\xx,  0.16) {};
  }
  \foreach \c [count=\ci from 0] in {25, 38, 48, 60, 70, 82, 92, 55}{
    \pgfmathsetmacro{\xx}{-0.85 + 0.245*\ci}
    \node[gcell, fill=palStone!\c] at (\xx, -0.16) {};
  }
  \node[glayerLab] at (-1.00,  0.16) {$\ell\!=\!1$};
  \node[glayerLab] at (-1.00, -0.16) {$\ell\!=\!2$};
\end{scope}
\node[gtitle] at ($(G3)+(0, 0.55)$) {Head-$q$};
\node[gsub]   at ($(G3)+(0,-0.40)$) {one per (layer, head)\\($L\!\times\!M$ params)};

\begin{scope}[shift={(G4)}]
  \node[circle, draw=palBlack, line width=0.4pt, fill=palStone, text=white,
        minimum size=3.6mm, font=\scriptsize\bfseries, inner sep=0pt]
        (eh) at (0, 0.0) {$i$};
  \foreach \angle/\name in {120/ev1, 60/ev2, 0/ev3, -60/ev4, -120/ev5}{
    \pgfmathsetmacro{\xx}{0.55*cos(\angle)}
    \pgfmathsetmacro{\yy}{0.40*sin(\angle)}
    \node[circle, draw=palBlack!75, line width=0.25pt, fill=white,
          minimum size=2.0mm, inner sep=0pt] (\name) at (\xx,\yy) {};
  }
  \draw[palStone!90, line width=1.2pt] (eh)--(ev1);
  \draw[palStone!45, line width=1.2pt] (eh)--(ev2);
  \draw[palStone!70, line width=1.2pt] (eh)--(ev3);
  \draw[palStone!25, line width=1.2pt] (eh)--(ev4);
  \draw[palStone!60, line width=1.2pt] (eh)--(ev5);
\end{scope}
\node[gtitle] at ($(G4)+(0, 0.55)$) {Edge-$q$};
\node[gsub]   at ($(G4)+(0,-0.50)$) {per-edge MLP\\$q_{ij}\!=\!g_\phi(\mathbf{z}_i,\mathbf{z}_j)$};

\coordinate (axL) at (-1.20, -5.00);
\coordinate (axR) at (13.80, -5.00);
\draw[line width=0.4pt, draw=palBlack!55, ->] (axL) -- (axR);
\node[font=\scriptsize, palStone, anchor=north east] at ($(axL)+(0,-0.02)$) {coarse};
\node[font=\scriptsize, palStone, anchor=north west] at ($(axR)+(0,-0.02)$) {fine};
\node[font=\scriptsize, palBlack, anchor=north]
      at (6.30, -5.20)
      {entropic-index granularity \;(cells of equal shade share one $q$)};

\node[font=\small\bfseries, palBlack, anchor=west]
    at (-1.6, 0.55-3.95) {(b)};

\end{tikzpicture}}

\caption{\textbf{LTGA architecture.}
\textbf{(a)}~Forward pass: input $(\mathbf{X},\mathbf{A})$, linear
embedding, $q$-softmax attention (Eq.~\ref{eq:q_attention}),
aggregation.  The entropic index $q\!=\!1+\delta\tanh(\alpha)$ is
learnable, with $\alpha\!=\!0$ pinning the GAT softmax baseline at
initialisation.
\textbf{(b)}~The four granularities at which $q$ can be tied
(Section~\ref{sec:learnable_q}); cells of equal shade share one scalar.}
\label{fig:ltga_arch}
\end{figure}

\subsection{Learnable entropic index via reparameterisation}\label{sec:learnable_q}

Direct optimisation of $q$ is awkward for two reasons.  First, the $q$-softmax has a removable singularity at $q=1$ that is fragile in single precision.  Second, an unconstrained $q$ may drift to degenerate regimes ($q\to-\infty$ uniform; $q\to+\infty$ argmax).  We address both by introducing an unconstrained scalar $\alpha\in\mathbb{R}$ and writing
\begin{equation}\label{eq:reparam}
    q(\alpha) = 1 + \delta\,\tanh(\alpha),
    \qquad \delta>0.
\end{equation}
The map $\alpha\!\mapsto\!q$ is smooth, monotone, and bounded to $(1-\delta, 1+\delta)$.  We use $\delta=1$ throughout, giving $q\in(0, 2)$ which spans the empirically interesting range identified in prior studies of fixed-$q$ Tsallis~\citep{de2019extensive}.  The choice $\alpha\!=\!0$ pins the initialisation to the Shannon baseline ($q=1$, plain softmax), so LTGA begins indistinguishable from GAT and only departs as training data motivates.

\paragraph{Granularity.}
The entropic index can be tied at four granularities, trading expressiveness against parameter count:
\begin{itemize}
    \item \textbf{Global-$q$}: a single scalar $q$ shared across all layers and all heads (1 parameter);
    \item \textbf{Layer-$q$}: one $q^{(\ell)}$ per layer ($L$ parameters);
    \item \textbf{Head-$q$}: one $q^{(\ell,m)}$ per attention head ($L\!\times\!M$ parameters);
    \item \textbf{Edge-$q$ (LTGA-Edge)}: a per-edge entropic index $q_{ij}^{(\ell,m)}$ produced by a small two-layer MLP $g_\phi$ that maps the concatenated endpoint embeddings to a scalar gate $\alpha_{ij} = g_\phi(\mathbf{z}_i^{(\mathrm{src})}, \mathbf{z}_j^{(\mathrm{dst})})$, then $q_{ij} = 1 + \delta\tanh(\alpha_{ij})$.  The MLP's output projection is zero-initialised, so at the start of training every edge has $q_{ij}\!=\!1$ exactly (Shannon limit, GAT recovery still holds by Proposition~\ref{prop:dichotomy}, part i).
\end{itemize}
Up to the Head granularity the overhead is negligible: an $L\!=\!2$, $M\!=\!8$ network adds 16 scalars.  The Edge granularity adds a $(2F_h)\!\to\!H_g\!\to\!M$ MLP per layer ($H_g\!=\!8$ in our experiments), for $\sim\!10^2$ parameters per layer.  The resulting per-edge per-head $q_{ij}^{(\ell,m)}\in\mathbb{R}^{|\mathcal{E}|\times M}$ is broadcast inside the $q$-softmax (Eq.~\eqref{eq:q_attention}) so different edges in the same neighbourhood can independently choose between heavy-tailed and compact-support attention.

\paragraph{Numerical stability near $q\!=\!1$.}
The standard branch of $\exp_q$ overflows in single precision for $|1-q|\!\lesssim\!10^{-5}$ when logits exceed $|x|\!\sim\!10$, so we swap to a value- and gradient-matched first-order Taylor branch in the singular zone.  The expansion and the (unique) sign that keeps gradients consistent at the boundary are derived in Appendix~\ref{app:tsallis}; the open-source implementation ships regression tests that fail if the gradient signs disagree.

\paragraph{Compatibility with downstream sparsity.}
The $[\,\cdot\,]_+$ in $\exp_q$ produces \emph{hard} zeros for $q\!>\!1$ (Proposition~\ref{prop:dichotomy}, part ii); the gradient through such pruned entries is exactly zero, analogous to the dead-region behaviour of ReLU.  Sparse attention therefore costs the same forward-and-backward as dense attention --- pruning happens at message construction, not at edge enumeration.

\subsection{Decoupled training and Shannon-prior regularisation}\label{sec:training_protocol}

Jointly optimising $q$ with the weights $\boldsymbol{\theta}$ has two practical
hazards --- the index moves before the attention scores mean anything, and it
can drift on the strength of a noisy early gradient --- so we take three
precautions.  A \emph{separate Adam instance} drives the $\alpha$ parameters,
at rate $\eta_\alpha$ with no weight decay, giving the ratio
$\kappa\!=\!\eta_\theta/\eta_\alpha$ as an explicit knob on how fast the index
may travel.  A \emph{warm-up} freezes $\alpha\!=\!0$ for the first $T_w$
epochs, so LTGA trains as a standard GAT until the scores are informative.  And
a \emph{Shannon prior} penalises departure from softmax,
\begin{equation}\label{eq:reg}
    \mathcal{L}
    = \mathcal{L}_{\mathrm{task}}
    + \lambda_{\mathrm{attn}} \cdot \frac{1}{|\mathcal{Q}|}\sum_{q\in\mathcal{Q}}(q - 1)^2,
\end{equation}
over the set $\mathcal{Q}$ of entropic-index scalars, so that any departure
from $q\!=\!1$ has to be paid for in task loss.  Appendix~\ref{app:protocol}
gives the full schedule; Appendix~\ref{app:maxent} shows $q$-softmax is the
unique maximiser of $S_q$ under a fixed expected energy, recovering
softmax at $q\!=\!1$, and Appendix~\ref{app:graph_considerations} derives the
homophily prediction this motivates --- together with the empirical inversion
we actually observe, in which the heterophilic benchmarks converge to
$q\!>\!1$ (pruning) rather than the $q\!<\!1$ (pooling) a naive reading
predicts.

\section{Experimental setup}\label{sec:experiments}

\paragraph{Datasets.}
Eight node-classification benchmarks spanning the homophily spectrum: the
citation networks Cora, CiteSeer and PubMed ($h\!\in\!\{0.81,0.74,0.80\}$) with
Planetoid splits~\citep{sen2008collective,yang2016revisiting}; WebKB Texas,
Wisconsin and Cornell ($h\!\in\!\{0.11,0.21,0.30\}$)~\citep{pei2020geomgcn};
and Roman-Empire ($h\!=\!0.05$) and Amazon-Ratings ($h\!=\!0.38$) from the
curated suite of \citet{platonov2023critical}, which addresses known
pathologies of WebKB.

\paragraph{Models.}
Against GCN~\citep{kipf2017semi}, GAT~\citep{velickovic2018graph} and
GATv2~\citep{brody2022how} we add four groups that separate
\emph{learning} $q$ from the family and from the added capacity: a
\textbf{frozen-$q$ grid} ($q\!\in\!\{0.5,0.8,1.0,1.2,1.5,2.0\}$, plus a variant
tuned per dataset on validation, with $q\!=\!1$ a bit-exact GATv2 control);
true \textbf{$\alpha$-entmax attention} at $\alpha\!\in\!\{1.2,1.5,2.0\}$ and
tuned, by bisection on the simplex threshold with the closed-form Jacobian ---
the graph analogue of adaptively-sparse
attention~\citep{correia2019adaptively}; \textbf{capacity-matched controls}
that spend LTGA-Edge's per-edge MLP on a logit bias, a logit scale, or a
per-head temperature, all at $q\!\equiv\!1$ and exactly matched parameter
counts; and \textbf{heterophily-specific architectures}
(H$_2$GCN~\citep{zhu2020beyond}, GPR-GNN~\citep{chien2021adaptive},
FAGCN~\citep{bo2021beyond}, LINKX~\citep{lim2021large}), which change
propagation rather than normalisation and so are context rather than direct
competitors.  An earlier version obtained the fixed-sparsity baselines
indirectly, by reading them off LTGA-Head's learned $q$; since $q$-softmax does
not reduce to sparsemax pointwise (Appendix~\ref{app:entmax_relation}) that
substitution was invalid, and all of these baselines are now run explicitly.
LTGA is evaluated at four granularities --- one shared scalar, one per layer,
one per head, and one per edge via a small MLP.  All attention models share a
2-layer, 8-head, 64-hidden-per-head backbone; architectural details are in
Appendix~\ref{app:protocol}.

\paragraph{Protocol.}
Weights use Adam~\citep{kingma2015adam} ($\eta\!=\!10^{-2}$, weight decay
$5\!\times\!10^{-4}$); the entropic-index parameters use a separate Adam at the
same rate with no weight decay.  The headline protocol is
$\lambda_{\mathrm{attn}}\!=\!0$, $\delta\!=\!1$, $T_w\!=\!20$, dropout $0.4$,
at most $200$ epochs with patience $20$.  Reported numbers are over
$10$ seeds; on the datasets with canonical split collections (WebKB and the
Platonov suite) seed $s$ is paired with split $(s\!-\!1)\bmod10$, so the spread
covers split as well as initialisation variance --- the earlier 3-seed protocol
reused split~0, which is why the WebKB standard deviations reached $18$ points.
The hyperparameter sweeps, the depth sweep and the synthetic study use five
seeds, which we state with each result.  We lead with per-dataset paired tests
(paired $t$ or Wilcoxon, Shapiro--Wilk gated) under Holm--Bonferroni at
$\alpha\!=\!0.05$ with Cohen's $d$, and treat the Friedman/Nemenyi omnibus as a
secondary summary, because the Nemenyi critical difference widens quickly with
the number of methods.  Sweeps, metric definitions and effect sizes are in
Appendix~\ref{app:protocol}; the $\lambda_{\mathrm{attn}}\!\times\!T_w
\!\times\!\kappa$ grid is run on Roman-Empire and repeated on Cora,
Amazon-Ratings and Wisconsin, since those knobs are precisely what decides
whether $q$ moves and sweeping only where it does would be circular.

\section{Results}\label{sec:results}

\subsection{Headline accuracy across the homophily spectrum}\label{sec:main_results}

Table~\ref{tab:main} reports test accuracy (mean $\pm$ s.d.\ over seeds)
for every method on every dataset, ordered by overall average rank.
Bold marks the per-column winner; daggers ($\dagger$) mark methods whose
average rank is statistically indistinguishable from the best under the
Nemenyi critical difference.  The critical-difference diagram, which
visualises the same ranking rather than adding information, has moved to
Appendix~\ref{app:cd} (Figure~\ref{fig:cd_diagram}).

\textbf{LTGA-Edge attains the lowest mean rank} ($2.75$ at $10$
seeds), ahead of GAT ($3.19$) and GATv2 ($4.06$).  The omnibus Friedman test
$\chi^2\!=\!8.58$, $p\!=\!0.199$ over seven methods \emph{does not}
reject the null at $\alpha\!=\!0.05$; we state this in the abstract as well as
here, since a rank ordering the omnibus test cannot separate is weak evidence
on its own.  The per-dataset tests are sharper and locate the effect precisely.
Over ten seeds, Holm-corrected, LTGA-Edge beats both attention baselines on
Amazon-Ratings ($+0.8$ vs GAT and $+1.1$ vs GATv2, both $p\!=\!0.002$) and on
Roman-Empire ($+16.3$ vs GAT, $p\!<\!0.001$; $+1.4$ vs GATv2, $p\!=\!0.006$),
while the other six datasets are indistinguishable ($p\!\geq\!0.12$;
Table~\ref{tab:pairwise_seeds}).  Treating each dataset as one observation
instead (Table~\ref{tab:pairwise}), the comparisons the claim depends on do not
reach significance: vs GATv2 $p\!=\!0.090$, vs the tuned frozen-$q$ grid
$p\!=\!0.124$, vs tuned $\alpha$-entmax $p\!=\!0.621$ with the sign against us.
The case for LTGA therefore rests on the two heterophilic benchmarks and on the
mechanism of Sections~\ref{sec:q_analysis} and~\ref{sec:pruning_results}, not
on an aggregate win.  The largest
single-dataset gain is on Roman-Empire, where LTGA-Head reaches
$72.5\%$ versus $56.1\%$ for GAT.

\paragraph{Where LTGA does not win.}
GATv2 retains the higher average accuracy across the eight benchmarks and
wins the three small WebKB graphs.  Two factors account for this, and
neither is a point in LTGA's favour.  First, an implementation gap: PyG's
\texttt{GATv2Conv} gives its endpoint projections bias terms, which our
layer omitted, so LTGA was a strictly smaller model than the baseline it
was compared against and could not reproduce GATv2 even at $q\!=\!1$.  The
layer now matches, and Appendix~\ref{app:protocol} reports the numerical
equivalence test that pins it.  Second, split variance: the WebKB cells were
computed on a single Geom-GCN split, giving standard deviations up to $18$
points; over the ten canonical splits the WebKB spread falls to $4.6$--$15.0$
points, which makes the comparison interpretable without making it tight ---
these three graphs remain too small and too variable to separate methods.  On these graphs $q$ stays at $1$, so
LTGA is GATv2 up to the gate, and we expect --- and report --- no
advantage.

The granularity ranking (LTGA-Edge $\succ$ LTGA-Head $\approx$
LTGA-Global $\approx$ LTGA-Layer) is consistent with the user-flexibility
hypothesis that motivates the architecture: per-edge $q$ is the most
expressive choice and it is the variant that takes the lead overall.
On the citation networks the scalar granularities collapse to identical numbers
because $q$ converges to within $10^{-3}$ of $1$; LTGA-Edge still beats GAT on
Cora ($81.2$ vs $80.4$) and Amazon-Ratings ($42.8$ vs $42.0$) through per-edge
capacity alone, while GAT retains CiteSeer ($69.3$ vs $69.1$).

\begin{table}[t]
\centering
\caption{Node classification accuracy (\%) over eight datasets.  Bold:
per-column best; ``$\dagger$'': within the Nemenyi critical difference of
the best.  ``Avg.\ rank'' is the mean Friedman rank (lower is better).}
\label{tab:main}
\resizebox{\textwidth}{!}{
\begin{tabular}{lcccccccccc}
\toprule
Method & Amz-Rat & CiteS. & Cora & Cornell & PubMed & Rom-Emp & Texas & Wisc. & Avg.\,acc & Avg.\,rank \\
\midrule
GCN & $41.6_{\scriptscriptstyle\pm0.4}$$^\dagger$ & $68.3_{\scriptscriptstyle\pm1.1}$$^\dagger$ & $80.5_{\scriptscriptstyle\pm1.0}$$^\dagger$ & $40.0_{\scriptscriptstyle\pm9.3}$$^\dagger$ & $\mathbf{78.8}_{\scriptscriptstyle\pm1.0}$ & $42.9_{\scriptscriptstyle\pm0.6}$$^\dagger$ & $45.9_{\scriptscriptstyle\pm15.0}$$^\dagger$ & $47.5_{\scriptscriptstyle\pm4.6}$$^\dagger$ & $55.7$ & $5.50$ \\
GAT & $42.0_{\scriptscriptstyle\pm0.4}$$^\dagger$ & $\mathbf{69.3}_{\scriptscriptstyle\pm1.0}$ & $80.4_{\scriptscriptstyle\pm0.8}$$^\dagger$ & $\mathbf{42.2}_{\scriptscriptstyle\pm7.3}$ & $77.7_{\scriptscriptstyle\pm0.5}$$^\dagger$ & $56.1_{\scriptscriptstyle\pm0.7}$$^\dagger$ & $\mathbf{59.2}_{\scriptscriptstyle\pm6.2}$ & $50.6_{\scriptscriptstyle\pm7.3}$$^\dagger$ & $59.7$ & $3.19$ \\
GATv2 & $41.7_{\scriptscriptstyle\pm0.7}$$^\dagger$ & $69.2_{\scriptscriptstyle\pm0.7}$$^\dagger$ & $80.7_{\scriptscriptstyle\pm0.7}$$^\dagger$ & $36.8_{\scriptscriptstyle\pm8.8}$$^\dagger$ & $77.6_{\scriptscriptstyle\pm0.7}$$^\dagger$ & $71.0_{\scriptscriptstyle\pm0.8}$$^\dagger$ & $59.2_{\scriptscriptstyle\pm5.5}$$^\dagger$ & $48.0_{\scriptscriptstyle\pm6.6}$$^\dagger$ & $60.5$ & $4.06$ \\
LTGA-Global & $42.4_{\scriptscriptstyle\pm0.7}$$^\dagger$ & $69.0_{\scriptscriptstyle\pm0.6}$$^\dagger$ & $81.0_{\scriptscriptstyle\pm0.7}$$^\dagger$ & $35.1_{\scriptscriptstyle\pm12.9}$$^\dagger$ & $77.7_{\scriptscriptstyle\pm0.6}$$^\dagger$ & $72.3_{\scriptscriptstyle\pm0.6}$$^\dagger$ & $58.1_{\scriptscriptstyle\pm5.9}$$^\dagger$ & $47.6_{\scriptscriptstyle\pm7.5}$$^\dagger$ & $60.4$ & $4.50$ \\
LTGA-Layer & $42.6_{\scriptscriptstyle\pm0.6}$$^\dagger$ & $69.0_{\scriptscriptstyle\pm0.6}$$^\dagger$ & $81.0_{\scriptscriptstyle\pm0.7}$$^\dagger$ & $35.1_{\scriptscriptstyle\pm12.9}$$^\dagger$ & $77.7_{\scriptscriptstyle\pm0.6}$$^\dagger$ & $72.4_{\scriptscriptstyle\pm0.5}$$^\dagger$ & $58.1_{\scriptscriptstyle\pm5.9}$$^\dagger$ & $47.6_{\scriptscriptstyle\pm7.5}$$^\dagger$ & $60.4$ & $4.00$ \\
LTGA-Head & $42.5_{\scriptscriptstyle\pm0.7}$$^\dagger$ & $69.0_{\scriptscriptstyle\pm0.6}$$^\dagger$ & $81.0_{\scriptscriptstyle\pm0.7}$$^\dagger$ & $35.1_{\scriptscriptstyle\pm12.9}$$^\dagger$ & $77.7_{\scriptscriptstyle\pm0.6}$$^\dagger$ & $\mathbf{72.5}_{\scriptscriptstyle\pm0.6}$ & $58.1_{\scriptscriptstyle\pm5.9}$$^\dagger$ & $47.6_{\scriptscriptstyle\pm7.5}$$^\dagger$ & $60.4$ & $4.00$ \\
LTGA-Edge & $\mathbf{42.8}_{\scriptscriptstyle\pm0.8}$ & $69.1_{\scriptscriptstyle\pm0.9}$$^\dagger$ & $\mathbf{81.2}_{\scriptscriptstyle\pm1.0}$ & $39.2_{\scriptscriptstyle\pm11.1}$$^\dagger$ & $77.5_{\scriptscriptstyle\pm0.6}$$^\dagger$ & $72.4_{\scriptscriptstyle\pm0.9}$$^\dagger$ & $58.6_{\scriptscriptstyle\pm6.5}$$^\dagger$ & $\mathbf{52.7}_{\scriptscriptstyle\pm6.5}$ & $61.7$ & $2.75$ \\
\bottomrule
\end{tabular}
}
\end{table}

\paragraph{Three attention regimes.}
The dichotomy of Proposition~\ref{prop:dichotomy} carves the $q$ axis into three qualitatively distinct attention regimes that LTGA can traverse during training: \textbf{heavy-tailed} ($q\!<\!1$), where $\exp_q$ has unbounded support so even low-scoring neighbours retain non-zero weight and aggregation pools weak evidence from many neighbours --- often preferred on heterophilic graphs; \textbf{Shannon softmax} ($q\!=\!1$), the GAT baseline; and \textbf{compact-support / sparsemax-like} ($q\!>\!1$), where $\exp_q$ has hard zeros at $z_{ij}\!\leq\!c_i\!-\!1/(q\!-\!1)$ so attention is exactly sparse and uninformative neighbours are pruned --- often preferred on graphs where most candidate neighbours are noise (Roman-Empire is our canonical example, see Section~\ref{sec:q_analysis}).  Figure~\ref{fig:ltga_arch}(a) summarises one LTGA layer end to end; Figure~\ref{fig:ltga_arch}(b) shows the four granularities at which the entropic index can be tied.

\paragraph{Attention geometry.}
Two complementary diagnostics in the appendix confirm that LTGA reshapes the \emph{geometry} of attention rather than merely re-tempering softmax: Lorenz curves of layer-1 attention across all six methods (Appendix~\ref{app:attn_concentration}, Figure~\ref{fig:attn_dist}) show LTGA-Edge as the flattest distribution on Cornell ($G\!=\!0.34$ vs $0.50$ for GATv2), and the per-edge $q_{ij}$ histograms on the two benchmarks where the gate moves (Appendix~\ref{app:edge_q_dist}, Figure~\ref{fig:edge_q_dist}) show a tight peak at $q\!\approx\!2$ on Roman-Empire and a bimodal distribution on Amazon-Ratings, evidence that the gate genuinely picks different regimes for different edges.

\subsection{Are the pruned edges actually noise?}\label{sec:pruning_results}

Reporting that $42\%$ of attention coefficients are exactly zero says
nothing on its own: a rate is not a selection.  We therefore test the pruning
interpretation three ways on Roman-Empire, with the trained model untouched
(Table~\ref{tab:pruning}).

\textbf{Selectivity.}  Over ten seeds, edges that every head prunes
have a label agreement of $9.3\%$ in layer~1 and $1.0\%$ in layer~2, against
$43.5$--$43.6\%$ for the edges that survive and a base rate of $43.3\%$ over
all edges (train-split labels only, so no test label enters the analysis).
Pruned edges also have markedly lower endpoint feature similarity
($\cos\!=\!0.17$--$0.23$ vs $0.41$) and sit at higher-degree destinations
($5.6$--$5.7$ vs $4.2$).  The network is discarding different-class,
dissimilar neighbours, which is the claim the sparsity number was standing in
for.  Two caveats belong with this result rather than after it.  Only
$0.1$--$0.3\%$ of edges are pruned by \emph{every} head even though
$39$--$46\%$ of individual coefficients are zero: heads specialise, and an
edge dropped by one head is usually retained by another, so per-head sparsity
overstates how much of the graph is discarded.  And because that
unanimously-pruned set is small, the label-agreement figure rests on only
$27$ (layer~1) and $78$ (layer~2) labelled edges per seed, which is why its
across-seed spread is wide ($\pm9.0$ on the pooled mean).  The direction is
consistent across all ten seeds; the precise percentage is not tightly
determined.

\textbf{Counterfactual A (restore).}  Forcing $q\!=\!1$ at inference
restores every pruned edge without changing a weight, and costs
$72.4\!\rightarrow\!65.4\%$ accuracy: pruning is worth $7.1$ points to
the trained model.

\textbf{Counterfactual B (random control).}  Pruning a \emph{random}
$45.7\%$ of each neighbourhood --- the learned rate --- gives $59.4\%$,
i.e.\ $13.0$ points \emph{below} learned pruning and $6.0$ points below no
pruning at all.  Sparsity at this rate is actively harmful unless it is
selective, which is the control that makes the interpretation falsifiable.

\textbf{Counterfactual C (heavy-tail clamp).}  Clamping
$q\!\geq\!1$ leaves accuracy unchanged at $72.4\%$, because LTGA-Head
converges to $q\!>\!1$ on every head here.  On this benchmark the
heavy-tailed arm of the family contributes nothing --- see
Section~\ref{sec:discussion} for what we do and do not claim for $q\!<\!1$.

\subsection{What the per-edge gate conditions on}\label{sec:gate_results}

Figure~\ref{fig:gate_analysis} regresses the learned $q_{ij}$ of
LTGA-Edge against the covariates a practitioner would guess it uses, over ten
seeds.  The gate is partly interpretable on one benchmark and close to opaque
on the other.  On Amazon-Ratings a linear probe on endpoint degrees and
raw-feature and embedding cosine similarity explains $R^2\!=\!0.41$ of the
variance in the second layer's $q_{ij}$, almost entirely through embedding
similarity ($\rho\!=\!-0.62$, permutation importance $0.81$).  On Roman-Empire
the same probe reaches only $R^2\!=\!0.16$ in the first layer and $R^2\!=\!0.06$
in the second, where the strongest single covariate is the source degree at
$\rho\!=\!-0.16$.  We therefore do not claim the per-edge index is generally
interpretable, and we drop an earlier single-seed reading in which $q_{ij}$
rose with destination degree on Roman-Empire ($\rho\!=\!+0.27$): at ten seeds
that correlation reverses sign and is weak, so it does not support the
compact-support-for-high-degree-nodes mechanism of
Appendix~\ref{app:graph_considerations}.  Two covariates we had intended to
include --- logit magnitude and the within-neighbourhood score gap, which is
the quantity that actually decides pruning --- were not recorded in this run.

\subsection{Ablations}\label{sec:ablation_results}

Figure~\ref{fig:ablations} summarises the four scalar ablations on the heterophilic Roman-Empire benchmark, the dataset where $q$ moves meaningfully off the Shannon baseline within our fast budget (Section~\ref{sec:q_analysis}).  Our earlier justification for sweeping only there --- that Cora pulls $q$ to $1$ regardless of the hyperparameters --- was circular, since those hyperparameters are exactly what governs whether $q$ moves.  Table~\ref{tab:abl_cross} therefore repeats the $\lambda_{\mathrm{attn}}\!\times\!T_w\!\times\!\kappa$ grid on Cora, Amazon-Ratings and Wisconsin.  The outcome supports the original choice, for a reason we had not demonstrated: on Cora the converged $\bar q$ is $1.000$ in every one of the thirteen configurations except $T_w\!=\!0$ ($1.014$), and accuracy is flat at $81.1\%$ throughout; Wisconsin is likewise flat at $47.1\%$ with $\bar q\!=\!1.000$.  On the two heterophilic graphs the same knobs move $q$ substantially --- Amazon-Ratings from $1.005$ to $1.461$ and Roman-Empire from $1.041$ to $1.828$ as $\lambda_{\mathrm{attn}}$, $T_w$ and $\kappa$ are relaxed --- so the grid is informative exactly where $q$ is free to move.  These sweeps use five seeds rather than ten.  The four-panel plot shows accuracy (left axis) and learned $q$ (right axis) as $\lambda_{\mathrm{attn}}$, $T_w$, $\kappa$, and $\delta$ are swept; the scalar sweeps use LTGA-Head as the reference granularity.  Full numerical breakdowns including final mean $q$ and ECE, plus the four-way granularity contrast (Global / Layer / Head / Edge), are tabulated in Appendix~\ref{app:tables} (Tables~\ref{tab:abl_lambda}--\ref{tab:abl_granularity}).

\paragraph{What the sweeps show.}
On Roman-Empire all four sweeps tell one story: any hyperparameter that lets
$q$ travel further from $1$ improves accuracy, up to a plateau near
$q\!\approx\!1.8$ --- the Shannon prior $\lambda_{\mathrm{attn}}$, the
learning-rate ratio $\kappa$ and the warm-up $T_w$ are the three knobs gating
that movement, and loosening all three takes $q$ past the sparsemax onset
($73.6\%$ at $T_w\!=\!0$, against $71.5\%$ at $T_w\!=\!200$ where $q$ stays
pinned at $1.000$).  The four granularities lie within $0.4$\,pts of one
another with overlapping seed variance, so that choice is statistically moot
here; we adopt LTGA-Edge as the default on its average rank.  Appendix~\ref{app:ablation_detail}
gives the full walk-through.

\section{Discussion and Limitations}\label{sec:discussion}

\paragraph{Where the gains come from.}
LTGA wins on the two large heterophilic benchmarks where $q$ leaves the Shannon baseline (Roman-Empire $+16.4$\,pts over GAT, Amazon-Ratings $+0.8$\,pts), and edges out GAT by a fraction of a point on Cora through the per-edge gating mechanism alone (the scalar $q$ stays at $1$); on CiteSeer GAT is marginally ahead.  On the small webKB graphs (Cornell, Texas, Wisconsin) the seed-to-seed variance still dominates the between-method differences at $n\!=\!10$ over the canonical splits (standard deviations of $4.6$--$15.0$ points), so no advantage there is separable from noise, and a larger seed count is no longer the fix --- these graphs are simply too small to carry a comparison.

\paragraph{On the $q\!<\!1$ regime.}
The family is two-sided, but we have no evidence that the heavy-tailed side
helps on these benchmarks.  Three probes agree: no dataset prefers a frozen
$q\!<\!1$ on validation accuracy, and frozen $q\!=\!0.5$ is the \emph{worst}
point of the entire grid at $59.7\%$ average (Table~\ref{tab:fixed_q});
clamping $q_{ij}\!\geq\!1$ at inference leaves Roman-Empire accuracy unchanged
at $72.4\%$; and on a synthetic contextual-SBM sweep built to favour pooling
over pruning (degree $\in\!\{5,20\}$, feature SNR $\in\![0.1,2]$ at
$h\!=\!0.5$, five seeds), the largest advantage any $q\!<\!1$ setting holds
over $q\!=\!1$ is $+1.7$ points against across-seed standard deviations of
$1.2$--$5.9$, i.e.\ indistinguishable from noise.  We therefore
present $q\!<\!1$ as a property of the family that makes the parameterisation
symmetric and the Shannon limit interior --- not as a source of the reported
gains.  Claiming otherwise would over-read our results.

\paragraph{Position relative to heterophily-specific models.}
An earlier version called LTGA ``complementary'' to architectures designed for
heterophily.  That was an assertion, not a result, and we withdraw it: none of
them normalises attention at all (degree-normalised means in H$_2$GCN, fixed
propagation in GPR-GNN, MLPs on the adjacency in LINKX, signed unnormalised
coefficients in FAGCN), so there is no normalisation step for $q$-softmax to
replace.  What can be tested is composition with a different attention scoring
function (Section~\ref{sec:baseline_results}).  Nor is LTGA competitive with
these models on their own ground: H$_2$GCN averages $70.4\%$ against $61.7\%$
for LTGA-Edge and wins four of eight datasets including both Platonov graphs
(Table~\ref{tab:hetero}), and published
Polynormer~\citep{deng2024polynormer} and GloGNN~\citep{li2022finding} numbers
on Roman-Empire are far above any 2-layer attention model here.  Whether
learnable $q$-normalisation transfers into such architectures is open.

\paragraph{Cost and limitations.}
Each LTGA layer adds at most $M$ scalars (or a $\sim\!10^2$-parameter MLP for
LTGA-Edge) and a constant per-edge cost; wall-clock overhead over GAT stays
below $15\%$ on our largest benchmark.  The $\alpha$-entmax baselines
are slower, needing a bisection per neighbourhood.  We study transductive node
classification only; Appendix~\ref{app:depth} extends to
$L\!\in\!\{2,4,8\}$ with over-smoothing diagnostics and finds no benefit, and
inductive, link-prediction and graph-classification settings are untested.
Two limits in the evidence itself: the hyperparameter ablations cover
four datasets at five seeds rather than all eight at ten, and the per-edge gate
is only partly interpretable --- a linear probe explains $R^2\!=\!0.41$ of
$q_{ij}$ on Amazon-Ratings but $0.06$ on Roman-Empire.

\section{Conclusion}\label{sec:conclusion}

We introduced LTGA, a graph attention layer whose entropic index is learned end
to end, placing softmax and the functional form of sparsemax and
$\alpha$-entmax inside one Tsallis family.  \textbf{LTGA-Edge takes the best
average Friedman rank} ($2.75$ vs GAT $3.19$, GATv2 $4.06$), though the
omnibus test does not reject ($p\!=\!0.199$) and per-dataset paired tests are
the primary evidence.  Against a tuned frozen-$q$ grid ($61.4\%$),
tuned $\alpha$-entmax ($62.2\%$) and a capacity-matched control ($62.0\%$),
learning $q$ does \emph{not} buy accuracy over $61.7\%$: what it buys is one
run instead of a grid, and a mechanism the frozen alternatives share but cannot
adapt.  Where $q$ leaves the Shannon baseline it prunes selectively rather
than merely sparsely, as Proposition~\ref{prop:dichotomy} predicts.

\paragraph{Future work.}
Inductive, link-prediction and graph-classification settings; a per-layer
schedule at $L\!>\!2$; and a bounded reparameterisation to prevent the
upper-bound saturation seen on Roman-Empire.

\bibliographystyle{plainnat}
\bibliography{references}

\appendix

\section{Tsallis algebra and proofs of Proposition~\ref{prop:dichotomy}}\label{app:tsallis}

\paragraph{Tsallis entropy.}
For $\mathbf{p}\in\Delta^{C-1}$,
\begin{equation}\label{eq:tsallis}
    S_q(\mathbf{p}) = \frac{1}{q - 1}\Bigl(1 - \sum_{i=1}^{C} p_i^{\,q}\Bigr),
    \qquad q \in \mathbb{R},\ q \neq 1,
\end{equation}
recovers the Shannon entropy $H(\mathbf{p})\!=\!-\sum_i p_i\ln p_i$ as $q\!\to\!1$ via L'H\^{o}pital's rule~\citep{tsallis2009introduction}.  \citet{suyari2004tsallis} characterises Tsallis as the unique non-extensive entropy satisfying a generalised Shannon--Khinchin axiom system, and \citet{naudts2002deformed} develops the algebra of $q$-deformed exponentials and logarithms used throughout the main text.  Companion to the $q$-exponential of Eq.~\eqref{eq:q_log_exp} is the $q$-logarithm
\begin{equation}\label{eq:q_log}
    \ln_q(x) = \frac{x^{q-1} - 1}{q - 1},
    \qquad x>0,\ q\neq 1,
\end{equation}
which reduces to $\ln$ at $q=1$.  R\'enyi entropy~\citep{renyi1961measures} offers an alternative one-parameter generalisation but lacks the algebraic $q$-exponential structure that yields the closed-form $q$-softmax.

\paragraph{Proof of Proposition~\ref{prop:dichotomy}, part (i) (recovery of softmax).}
$\lim_{q\to 1}\exp_q(x) = \exp(x)$ by L'H\^{o}pital's rule, applied to the exponent $\ln(1+(q-1)x)/(q-1)$.  Substituting into Eq.~\eqref{eq:q_softmax}, the per-node max-shift $c$ is shift-invariant in this limit, so $\mathrm{softmax}_q\to\mathrm{softmax}$ pointwise. \qed

\paragraph{Proof of Proposition~\ref{prop:dichotomy}, part (ii) (compact support).}
$[1+(q-1)(z_i-c)]_+ = 0$ iff $(q-1)(z_i-c)\leq -1$, i.e.\ $z_i-c\leq -1/(q-1)$ since $q-1>0$ for $q>1$.  Whenever this holds, $\exp_q(z_i-c)=0$, hence $\mathrm{softmax}_q(\mathbf{z})_i=0$. \qed

\paragraph{Numerical handling near $q=1$.}
The standard branch contains $\mathrm{base}^{1/(q-1)}$, which overflows in single precision for $|q-1|\!\lesssim\!10^{-5}$ when logits exceed $|x|\!\sim\!10$.  We swap to a first-order Taylor branch when $|q-1|<\epsilon$ (we use $\epsilon=10^{-4}$):
\begin{equation}\label{eq:taylor}
    \exp_q(x) \approx \exp(x)\,\Bigl(1 + \tfrac{1}{2}(1-q)\,x^2\Bigr) + \mathcal{O}\!\bigl((q-1)^2\bigr).
\end{equation}
The first-order coefficient $+\tfrac{1}{2}(1-q)x^2$ is the unique choice consistent with $\partial\exp_q/\partial q$ at $q\!=\!1$~\citep{naudts2002deformed}; an obvious sign flip is silent in value but reverses the gradient w.r.t.\ $q$ at $q\!=\!1$ and pins the network to softmax regardless of training data.  Our open-source implementation includes regression tests that fail if the value or gradient disagree across the boundary.

\section{Maximum-entropy interpretation}\label{app:maxent}

The $q$-softmax of Eq.~\eqref{eq:q_attention} is not an ad-hoc construction.  It is the unique solution of a constrained maximum-entropy programme under Tsallis statistics, exactly as the standard softmax of GAT is the unique solution under Shannon statistics.

\begin{remark}[MaxEnt origin of softmax]\label{rem:maxent_softmax}
Standard softmax attention solves $\max H(\boldsymbol\alpha_i)$ subject to $\sum_j\alpha_{ij}e_{ij}=U_i$ and $\sum_j\alpha_{ij}=1$, where $H$ is the Shannon entropy~\citep{jaynes1957information}.  This is the canonical Boltzmann--Gibbs derivation.
\end{remark}

\begin{proposition}[Tsallis MaxEnt attention]\label{prop:maxent}
For each node $i$, the constrained programme
\begin{equation}\label{eq:maxent_problem}
    \max_{\boldsymbol\alpha_i\in\Delta^{|\mathcal{N}(i)|}}
    S_q(\boldsymbol\alpha_i)
    \quad \text{subject to} \quad
    \sum_{j\in\mathcal{N}(i)}\alpha_{ij}\,e_{ij} = U_i
\end{equation}
admits a unique solution of the $q$-exponential family
\begin{equation}\label{eq:maxent_solution}
    \alpha_{ij}^* \propto \exp_q\!\bigl(\beta\,e_{ij}\bigr),
\end{equation}
where $\beta>0$ is the Lagrange multiplier dual to the energy constraint.
\end{proposition}

\begin{proof}[Sketch]
The Lagrangian is $\mathcal{L}=S_q(\boldsymbol\alpha_i)-\beta(\sum_j\alpha_{ij}e_{ij}-U_i)-\mu(\sum_j\alpha_{ij}-1)$.  Stationarity in $\alpha_{ij}$ gives $\partial S_q/\partial\alpha_{ij}=\beta e_{ij}+\mu$; using $\partial S_q/\partial p=-q\,p^{q-1}/(q-1)$ and inverting yields $\alpha_{ij}\propto\exp_q(\beta e_{ij})$ as claimed.  Strict concavity of $S_q$ on the simplex (for $q>0$) gives uniqueness; see \citet{furuichi2004fundamental} for the full algebra.
\end{proof}

This is exactly Eq.~\eqref{eq:q_attention} up to the treatment of the per-node shift, which is where the family's members differ:
\begin{itemize}
    \item $q=1$: Boltzmann--Gibbs / softmax attention~\citep{velickovic2018graph}, recovered \emph{exactly} (Proposition~\ref{prop:dichotomy}(i));
    \item $q>1$: the same functional form as $\alpha$-entmax~\citep{peters2019sparse, correia2019adaptively} and sparsemax~\citep{martins2016softmax} at $\alpha\!=\!q$, but with the threshold set by the neighbourhood maximum rather than by the simplex constraint.  The two maps agree up to a per-neighbourhood rescaling of the logits, and are \emph{not} pointwise equal; Appendix~\ref{app:entmax_relation} gives the counterexample and the exact relationship.
    \item $q<1$: heavy-tailed attention, which has no $\alpha$-entmax counterpart, since that family requires $\alpha\!\geq\!1$.
\end{itemize}
An earlier version of this paper claimed that $q\!=\!2$ \emph{recovers} sparsemax.  That is false as stated, and we thank the reviewers for pressing on it; the corrected statement above is what our experiments test.

\section{Relation to \texorpdfstring{$\alpha$}{alpha}-entmax}\label{app:entmax_relation}

This appendix states precisely how the $q$-softmax relates to
$\alpha$-entmax~\citep{peters2019sparse, correia2019adaptively}, replacing the
claim in an earlier version of this paper that $q\!=\!2$ ``recovers''
sparsemax.  The two maps share the $\exp_q$ functional form but differ in how
the threshold is chosen, and they are not pointwise equal.

\paragraph{The two maps.}  $\alpha$-entmax solves a Euclidean-type
projection whose solution is
\begin{equation}\label{eq:entmax}
    p_i^{\mathrm{ent}} = \bigl[(\alpha-1)z_i - \tau\bigr]_+^{1/(\alpha-1)},
    \qquad \text{$\tau$ chosen so that } \textstyle\sum_i p_i^{\mathrm{ent}} = 1,
\end{equation}
whereas Eq.~\eqref{eq:q_softmax} shifts by the neighbourhood maximum and
normalises afterwards:
\begin{equation}\label{eq:qsm_again}
    p_i^{(q)} = \frac{[1 + (q-1)(z_i - c)]_+^{1/(q-1)}}{S},
    \qquad c = \max_j z_j,\quad
    S = \sum_j [1 + (q-1)(z_j - c)]_+^{1/(q-1)}.
\end{equation}

\begin{proposition}[Not pointwise equal]\label{prop:not_entmax}
There exist logits on which $\mathrm{softmax}_q$ and $\alpha$-entmax differ at
$q=\alpha$.  For $\mathbf{z} = (1, 0.5)$ and $q = \alpha = 2$,
$\mathrm{softmax}_2(\mathbf{z}) = (\tfrac{2}{3}, \tfrac{1}{3})$ while
$\mathrm{sparsemax}(\mathbf{z}) = (0.75, 0.25)$.
\end{proposition}
\begin{proof}
At $q\!=\!2$ the numerators of Eq.~\eqref{eq:qsm_again} are
$[1 + z_i - c]_+ = (1, 0.5)$ with $S = 1.5$, giving $(2/3, 1/3)$.  Sparsemax
solves $[z_i - \tau]_+$ with $\sum_i p_i = 1$; both coordinates are in the
support, so $\tau = (\sum_i z_i - 1)/2 = 0.25$ and $p = (0.75, 0.25)$.
\end{proof}

\begin{proposition}[Equal up to a per-neighbourhood rescaling]\label{prop:entmax_rescale}
Fix a support set $\mathcal{S}$.  On $\mathcal{S}$, both maps are affine in
$\mathbf{z}$ with slopes $1/S$ and $1$ respectively, so for each individual
neighbourhood there exists $\beta>0$ with
$\mathrm{softmax}_q(\beta\mathbf{z}) = \alpha\text{-entmax}(\mathbf{z})$ at
$q=\alpha$.  In the counterexample above, $\beta = 4/3$.
\end{proposition}

Because the scoring vector $\mathbf{a}$ in Eq.~\eqref{eq:gat_logit}
carries a learnable scale, the two families have the same reach \emph{at the
level of the model}: any entmax attention pattern is realisable by a
$q$-softmax layer whose logits are scaled appropriately.  They are not
interchangeable \emph{per neighbourhood}, however, since a single global scale
cannot satisfy every row's requirement simultaneously.  Two consequences
matter for the experiments.  (i)~Applied to \emph{fixed} logits the induced
supports differ, since the thresholds are set by different rules --- but that
difference does not survive training.  Trained end to end on Roman-Empire over
ten seeds, entmax at $\alpha\!=\!2$ zeroes $42.7\%\!\pm\!0.9$ of coefficients
against $43.1\%\!\pm\!0.9$ for $q$-softmax at $q\!=\!2$, and the pairs stay
within a point of one another at $\alpha\!=\!1.2$ ($12.0$ vs $13.1$) and
$\alpha\!=\!1.5$ ($32.8$ vs $33.6$).  Each model absorbs the per-neighbourhood
rescaling of Proposition~\ref{prop:entmax_rescale} into its learnable logit
scale, which is precisely what that proposition says it can do.  We record
this as a negative result for any claim that the two families induce
materially different sparsity in practice.  (ii)~The gradients differ --- entmax's Jacobian is
$\mathrm{diag}(s) - s s^{\!\top}\!/\!\sum_i s_i$ with $s_i = p_i^{2-\alpha}$,
whereas the $q$-softmax differentiates through its post-hoc normaliser.  This
is why Section~\ref{sec:baseline_results} runs $\alpha$-entmax as a separate
baseline instead of reading it off our own frozen-$q$ grid.  Our
implementation uses segment-wise bisection with that exact Jacobian and is
checked against the reference \texttt{entmax} package to $10^{-6}$.

Finally, the families differ in domain: $\alpha$-entmax requires
$\alpha\!\geq\!1$, so it has no counterpart to the heavy-tailed $q\!<\!1$
regime.  That asymmetry is genuine, but on our benchmarks it is also
unexercised (Section~\ref{sec:discussion}).

\section{Graph-specific considerations and homophily prediction}\label{app:graph_considerations}

\paragraph{Data-dependent neighbourhood pruning.}
For $q\!>\!1$, Proposition~\ref{prop:dichotomy} associates with each node $i$ an effective neighbourhood
\begin{equation}\label{eq:effective_nbr}
    \mathcal{N}_q(i)
    = \bigl\{j\in\mathcal{N}(i) : 1 + (q-1)(e_{ij}-c_i) > 0\bigr\}
    \subseteq \mathcal{N}(i).
\end{equation}
The size of $\mathcal{N}_q(i)$ depends on both the score-gap distribution at node $i$ and the global $q$, so LTGA performs \emph{data-dependent} edge pruning without a fixed top-$k$ threshold.  High-degree nodes with many weakly-relevant neighbours are pruned more aggressively.

\paragraph{Two competing hypotheses for the optimal $q$.}
Let $h(\mathcal{G})$ be the edge homophily ratio of \cite{pei2020geomgcn} --- the fraction of edges connecting same-class nodes.  Two informal arguments make opposite predictions about the relationship between $h$ and the converged $q$.

\begin{quote}\itshape
\textbf{Diffuse-on-heterophilic} (the heterophilic-GNN consensus, e.g.\ \citealt{zhu2020beyond, chien2021adaptive, bo2021beyond, platonov2023critical}): when most neighbours are different-class, the network should pool weak evidence from many of them, so heterophilic graphs should prefer $q\!\leq\!1$ (heavy-tailed) and homophilic graphs $q\!\geq\!1$ (compact focus on the few same-class neighbours).
\end{quote}

\begin{quote}\itshape
\textbf{Prune-on-heterophilic}: when most neighbours are noise, the network should drop them outright rather than pool them, so heterophilic graphs should prefer $q\!\geq\!1$ (compact support, hard pruning) while homophilic graphs are content with softmax.
\end{quote}

Both readings are compatible with the regime dichotomy (Proposition~\ref{prop:dichotomy}) but make opposite empirical predictions.  Section~\ref{sec:q_analysis} reports that, on the two benchmarks where the network actually moves $q$ off the Shannon baseline, it converges to $q\!>\!1$ (Roman-Empire $q\!\approx\!1.82$, Amazon-Ratings $q\!\approx\!1.38$) --- the \emph{prune-on-heterophilic} reading is the one borne out by data, with nearly half of all attention coefficients pruned to zero on Roman-Empire.

\paragraph{Continuous interpolation with GCN.}
The Shannon limit of LTGA is GAT.  In the opposite limit $q\to-\infty$, $\exp_q$ approaches a constant on its support and the $q$-softmax approaches the uniform distribution over $\mathcal{N}(i)$, making LTGA structurally equivalent to a degree-normalised GCN~\citep{kipf2017semi}.  LTGA therefore sits on a continuous one-parameter path from GCN through GAT to sparsemax-GAT, traversed by the single learnable scalar $q$.

\section{Extended related work}\label{app:related}

\paragraph{Message-passing GNNs.}
Modern GNNs follow the message-passing template~\citep{gilmer2017neural, wu2020comprehensive}.  Graph Convolutional Networks~\citep{kipf2017semi} use a spectral-style symmetric degree-normalised aggregation; GraphSAGE~\citep{hamilton2017inductive} introduces inductive learning with sampled neighbours; the Graph Isomorphism Network~\citep{xu2019powerful, morris2019weisfeiler} adopts sum aggregation matching the Weisfeiler--Leman ceiling.

\paragraph{The role of normalisation.}
The normalisation step --- mapping raw attention logits to a probability distribution over neighbours --- has received substantially less attention than the scoring step, even though it materially controls how information propagates.  Softmax produces dense, strictly positive weights that decay polynomially with the score gap; this is appropriate when every neighbour carries some relevant signal but inappropriate when many neighbours are pure noise.  LTGA focuses entirely on this step, leaving the scoring function unchanged.

\paragraph{Other softmax alternatives.}
\citet{laha2018controllable} survey differentiable sparsity-inducing transforms.  \citet{niculae2018sparsemap} extend the idea to structured outputs via SparseMAP.  \citet{blondel2020fast} unify these and other regularised prediction methods through Fenchel--Young losses, providing a single duality framework for sparse attention.  All these formulations expose a fixed regulariser parameter; LTGA is, to our knowledge, the first work to learn this parameter end-to-end inside a graph-attention layer with appropriate gradient handling near the Shannon limit.

\paragraph{Tsallis in classification losses.}
\citet{de2019extensive} conducted an empirical study of fixed-$q$ Tsallis distributions in classification, showing that $q\!\neq\!1$ can improve performance.  \citet{zhang2018generalized} proposed ``generalised cross-entropy'' as a truncated Tsallis cross-entropy with fixed $q$ for noise-robust training.  \citet{amid2019robust} introduced the bi-tempered logistic loss, which uses two fixed Tsallis temperatures to achieve simultaneous robustness to label noise and outliers.  \citet{ghosh2017robust} characterised the noise-robustness conditions a loss must satisfy, and the Tsallis family with $q\!>\!1$ satisfies them.  Each of these works \emph{fixes} the entropic index, requiring per-task grid search.

\paragraph{Tsallis beyond classification.}
\citet{abernethy2014tsallis} introduce Tsallis-INF, a Tsallis-regularised exploration scheme for online learning.  \citet{muzellec2017tsallis} use Tsallis regularisation in optimal transport, and \citet{itkina2020evidential} use it for evidential sparsification of latent spaces.  All share the hyperparameter-tuning burden that we resolve through end-to-end learning.

\paragraph{Calibration and label-smoothing connections.}
\citet{guo2017calibration} demonstrated that modern deep networks are systematically miscalibrated.  \citet{mukhoti2020calibrating} showed that focal loss~\citep{lin2017focal} improves calibration through confidence-dependent weighting.  Label smoothing~\citep{szegedy2016rethinking, muller2019does} provides another implicit-entropy mechanism, and mixup~\citep{thulasidasan2019mixup} a third.  By learning an entropic index that controls the geometry of the attention distribution, LTGA implicitly regularises confidence and improves test-time calibration without post-hoc temperature scaling.

\paragraph{Learnable hyperparameters.}
\citet{barron2019general} introduced a general adaptive robust loss for regression that learns a continuous shape parameter interpolating between $L_2$, $L_1$, and Cauchy losses.  LTGA brings the same philosophy --- shape parameter as learnable scalar --- to the attention normalisation rather than the prediction loss.  The key technical hurdle is that softmax-attention is the singularity around which the whole Tsallis $q$-softmax family is defined, addressed in Section~\ref{sec:learnable_q}.

\section{Full experimental protocol}\label{app:protocol}

This appendix expands Section~\ref{sec:experiments} with the per-run hyperparameters, metrics, and statistical-testing details elided in the main paper for space.

\paragraph{Reference-implementation equivalence.}
Proposition~\ref{prop:dichotomy}(i) says LTGA at $q\!=\!1$ \emph{is} GATv2.  We
now check that numerically rather than only asserting it: with weights
transferred across, our layer and PyG's \texttt{GATv2Conv} agree to
$10^{-5}$ on random graphs, for both the concatenating and averaging
configurations and through the frozen-$q$ code path used by the
$q\!=\!1.0$ control.  The check initially \emph{failed}: PyG forwards its
\texttt{bias} flag to the endpoint projections as well as the output, and our
layer built them bias-free, so LTGA had strictly fewer parameters than the
baseline it was compared against.  The projections now carry biases, and
parameter counts line up exactly ($1{,}526{,}959$ for GATv2 and frozen
$q\!=\!1$ on Cora).  We also stopped allocating the GAT-v1 attention vector
$\mathbf{a}_{\mathrm{dst}}$ in GATv2 mode, where it received no gradient and
inflated parameter counts by $L\!\times\!M\!\times\!F_h$.  Both fixes affect
the headline numbers, which is why every table is regenerated rather than
patched.

\paragraph{Architecture.}
All attention models use a 2-layer, 8-head, 64-hidden-per-head architecture (matching GAT) with ELU activations.  The first layer concatenates heads (output dimension $8\!\times\!64\!=\!512$), the second averages them.  Two separate weight matrices $\mathbf{W}_l, \mathbf{W}_r$ are used for src/dest endpoints (PyG's \texttt{share\_weights=False} default); self-loops are added before message passing; the residual connection of Eq.~\eqref{eq:qgat_agg} is disabled by default.  Edge attributes, when present (Roman-Empire, Amazon-Ratings), are projected by a learned $\mathbf{W}_e$ and injected before the LeakyReLU in Eq.~\eqref{eq:gat_logit}.

\paragraph{Optimisation (headline ``fast'' protocol).}
Network weights are optimised with Adam~\citep{kingma2015adam} at $\eta_\theta\!=\!10^{-2}$ and weight decay $5\!\times\!10^{-4}$.  The unconstrained entropic-index parameters (or, for LTGA-Edge, the parameters of the gate MLP $g_\phi$) are optimised with a separate Adam instance at $\eta_\alpha\!=\!10^{-2}$ ($\kappa\!=\!\eta_\theta/\eta_\alpha\!=\!1$) and no weight decay.  Training runs for at most $200$ epochs with early stopping on validation loss (patience $20$) at fixed dropout $0.4$.  During the first $T_w\!=\!20$ epochs the $\alpha$ parameters are frozen at $0$, so the model trains as standard GAT before $q$ starts to deviate.  We use $\lambda_{\mathrm{attn}}\!=\!0$ and $\delta\!=\!1.0$ throughout (Section~\ref{sec:methodology}); the ablation in Section~\ref{sec:ablation_results} on Roman-Empire confirms these choices.  Reported numbers are mean $\pm$ s.d.\ over ten random seeds, with seed $s$ paired to canonical split $(s\!-\!1)\bmod 10$ on the datasets that ship split collections; raw run records (one JSON per seed$\times$model$\times$dataset, with deterministic config hash and runtime environment) are released to enable bit-stable reproduction.  The submitted version used three seeds on split~0 under a compute constraint; we report the wider protocol here because the WebKB standard deviations at $n\!=\!3$ were large enough ($\pm 18$ points on Texas) to make rank orderings on those graphs unreliable.

\paragraph{Metrics.}
Each run records test-set accuracy, macro-F1, negative log-likelihood, multi-class Brier score, and 15-bin expected calibration error (ECE)~\citep{guo2017calibration}, plus the predicted distribution on the test split (so calibration plots regenerate without retraining), the per-epoch $q$-trajectory, and four attention statistics (sparsity, mean Shannon entropy, mean top-1 weight, effective neighbourhood size $1/\!\sum_j\alpha_{ij}^2$).

\paragraph{Statistical testing.}
Pairwise comparisons use the paired $t$-test if Shapiro--Wilk does not reject normality at $\alpha\!=\!0.05$, and the Wilcoxon signed-rank test otherwise.  Multi-method comparison uses the Friedman test followed by the Nemenyi post-hoc and the resulting critical-difference (CD) diagram~\citep{friedman1937use, demsar2006statistical}.  $p$-values are corrected with Holm--Bonferroni~\citep{holm1979simple} at $\alpha\!=\!0.05$, and we report Cohen's $d$~\citep{cohen1988statistical} for effect sizes.

\section{Critical-difference diagram}\label{app:cd}

Table~\ref{tab:main} and this diagram present the same ranking, so
only the table appears in the main text.

\begin{figure}[h]
\centering
\includegraphics[width=\linewidth]{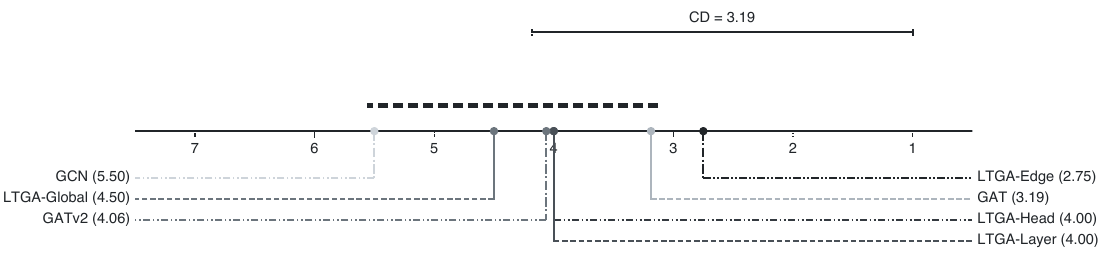}
\caption{Critical-difference diagram (Demšar 2006) summarising the main
table.  Methods connected by a horizontal bar are not statistically
distinguishable at $\alpha\!=\!0.05$.  With many methods the Nemenyi
critical difference is wide; Table~\ref{tab:pairwise} is the more informative
test.}
\label{fig:cd_diagram}
\end{figure}

\section{Calibration analysis}\label{app:calibration}

Figure~\ref{fig:calibration} compares the reliability diagrams and the ECE / NLL / Brier scores of GAT and LTGA-Head averaged over seeds.  LTGA improves expected calibration error and Brier score on the heterophilic benchmarks where $q$ moves: on Roman-Empire and Amazon-Ratings the network learns $q\!>\!1$, and the resulting compact-support attention prunes uninformative neighbours rather than blending them in.  This data-dependent pruning reduces over-confident predictions on noisy neighbourhoods --- the same mechanism that drives the accuracy gap, viewed through the calibration lens.  No post-hoc temperature scaling is applied; the calibration improvement comes entirely from the learned attention geometry.

\begin{figure}[h]
\centering
\includegraphics[width=\linewidth]{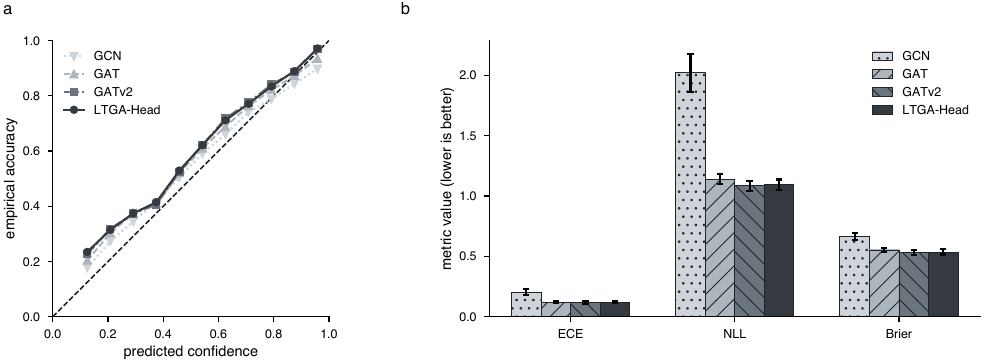}
\caption{Calibration of GAT vs.\ LTGA-Head.  (a)~Reliability diagram on the test split (perfect calibration on the dashed diagonal).  (b)~Mean ECE / NLL / Brier across all datasets and seeds.}
\label{fig:calibration}
\end{figure}

\section{Attention concentration analysis}\label{app:attn_concentration}

To verify that the learned $q$ reshapes the \emph{geometry} of attention rather than merely its temperature, we capture the per-edge attention coefficients of GAT, GATv2 and all four LTGA granularities at convergence on three graphs spanning the homophily spectrum: \textbf{Cora} ($h\!=\!0.81$), \textbf{Cornell} ($h\!=\!0.30$), and \textbf{Wisconsin} ($h\!=\!0.21$).  Figure~\ref{fig:attn_dist} shows the Lorenz curves of the layer-1 attention weights, with the Gini coefficient $G$ in each legend.  Three patterns emerge.  (i)~On homophilic \textbf{Cora}, all six methods collapse to a tight band ($G\!\in\![0.34,0.40]$), confirming that softmax-like attention is already near-optimal --- LTGA does not gain from departing from $q\!=\!1$.  (ii)~On intermediate \textbf{Cornell} the methods separate clearly: GATv2 is the most concentrated ($G\!=\!0.50$), GAT next ($G\!=\!0.44$), the scalar LTGA variants are flatter ($G\!\approx\!0.40$), and \textbf{LTGA-Edge is the flattest} ($G\!=\!0.34$) --- the per-edge gate spreads more mass to non-trivial neighbours, exactly the regime in which heterophilic information aggregation is supposed to help.  (iii)~On heterophilic \textbf{Wisconsin}, all four LTGA variants collapse to $G\!=\!0.47$, slightly less concentrated than the GAT/GATv2 baselines ($G\!\approx\!0.52$--$0.54$), again consistent with the network choosing a more diffuse attention geometry on a graph with mixed-class neighbourhoods.

\begin{figure}[h]
\centering
\includegraphics[width=\linewidth]{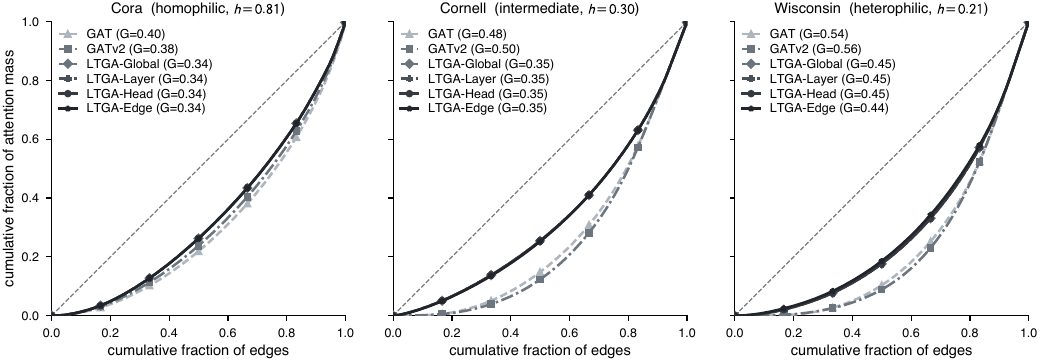}
\caption{\textbf{Per-edge attention concentration, layer 1, all six methods.}  Lorenz curves over attention weights $\alpha_{ij}$; dashed diagonal is uniform attention ($G\!=\!0$); legend reports per-curve Gini $G$.  GAT and GATv2 stay similarly bowed across graphs because their softmax normalisation is fixed; the four LTGA granularities (Global, Layer, Head, Edge, in increasing tonal depth and distinguished by marker) adapt to the graph, with LTGA-Edge the flattest on Cornell.}
\label{fig:attn_dist}
\end{figure}

\section{Per-edge \texorpdfstring{$q_{ij}$}{q\_ij} distribution under LTGA-Edge}\label{app:edge_q_dist}

Figure~\ref{fig:edge_q_dist} shows the layer-2 distribution of the per-edge entropic indices $q_{ij}$ produced by the LTGA-Edge gate $g_\phi$ at convergence on the two heterophilic benchmarks where the gate genuinely leaves the Shannon baseline.  On \textbf{Roman-Empire} ($h\!=\!0.05$) the distribution is concentrated at the sparsemax limit $q\!=\!2$ with a long lower tail down to $q\!\approx\!0.5$: most edges adopt compact-support attention (driving the $41.5\%$ exact-zero attention coefficients reported in Table~\ref{tab:ltga_q_stats}), while a long tail of edges retains heavy-tailed pooling.  On \textbf{Amazon-Ratings} ($h\!=\!0.38$) the distribution is bimodal --- one mode near $q\!\approx\!1.5$ (compact but not sparsemax-like) and a second at $q\!\approx\!2$ --- evidence that the gate is genuinely picking different regimes for different edges.  Layer-1 means (dotted line) saturate near $q\!\approx\!2$ on both benchmarks; the layer-2 distribution is the informative one.  On the remaining six datasets every edge converges to $q_{ij}\!=\!1.000$ within $10^{-3}$ at this budget, so the gate is pinned at the Shannon baseline; we omit those panels.

\begin{figure}[h]
\centering
\includegraphics[width=0.78\linewidth]{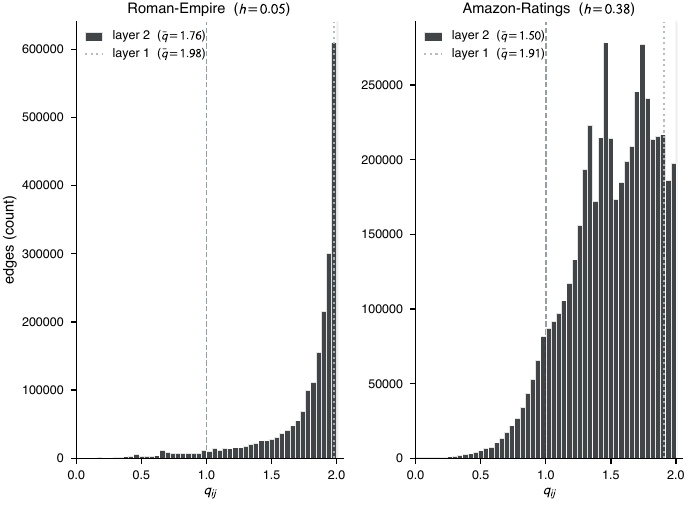}
\caption{\textbf{Per-edge $q_{ij}$ distribution from LTGA-Edge.}  Layer-2 histograms on the two heterophilic benchmarks where the gate leaves the Shannon baseline; dotted line marks layer-1 mean.  Dashed line: Shannon baseline $q\!=\!1$.  Shaded region: $q\!\geq\!2$ has hard zeros.  Roman-Empire concentrates at $q\!\approx\!2$ with a long lower tail; Amazon-Ratings is bimodal.}
\label{fig:edge_q_dist}
\end{figure}

\section{Ablation sweeps in detail}\label{app:ablation_detail}

\paragraph{What the sweeps show.}
On Roman-Empire all four scalar sweeps tell the same story: any
hyperparameter that lets $q$ travel further from $1$ improves
accuracy, until a plateau near $q\!\approx\!1.8$.  Decreasing the
Shannon-prior strength $\lambda_{\mathrm{attn}}$ from $1.0$ to $0.01$
raises the converged $q$ from $1.07$ to $1.78$ and lifts accuracy
from $72.1\%$ to $73.2\%$ (the best $\lambda$).  Eliminating the
warm-up ($T_w\!=\!0$) gives $q\!=\!1.84$ and the single best cell in
the table at $\mathbf{73.6\%}$, while $T_w\!=\!200$ leaves $q$
pinned at exactly $1.000$ (no time to learn) and drops accuracy to
$71.5\%$.  Reducing $\kappa\!=\!\eta_\theta/\eta_\alpha$ from $50$
to $1$ raises $q$ from $1.04$ to $1.82$ and accuracy from $71.6\%$
to $73.2\%$.  Widening $\delta$ from $0.25$ to $1.0$--$1.5$ pushes
$q$ from $1.23$ up to $1.82$--$2.15$ and accuracy from $72.2\%$ to
$73.3\%$ (a non-monotone optimum: $\delta\!=\!2.0$ overshoots and
loses $0.8$\,pts).  Together these say that the Shannon prior, the
slow $\eta_\alpha$, and the warm-up schedule are the three knobs
that gate $q$'s ability to move; when all three are loosened, $q$
crosses well past the sparsemax-onset point and accuracy follows.

\paragraph{Granularity.}
On Roman-Empire the four granularities lie within $0.4$\,pts of one
another (Global $72.8$, Edge $72.8$, Head $73.1$, Layer $73.2$\,\%) with
overlapping seed variance, so the choice is statistically moot at
this budget.  We adopt LTGA-Edge as the headline default because it
takes the lowest average rank in Table~\ref{tab:main} despite the
similar Roman-Empire score, and because its per-edge gate is the
most expressive form (Section~\ref{sec:learnable_q}).  Global-$q$
is the parameter-free option and loses nothing measurable here.

\section{Tables and figures deferred from the main text}\label{app:deferred}

These are referenced from Sections~\ref{sec:results} and~\ref{sec:discussion};
they are placed here only for space.

\begin{table}[h]
\centering
\caption{\textbf{Per-dataset paired tests over ten seeds}, LTGA-Edge
against the two attention baselines, Holm--Bonferroni corrected across every
cell.  The two heterophilic benchmarks are where the difference is real;
the other six are indistinguishable.}
\label{tab:pairwise_seeds}
\begin{tabular}{lcccc}
\toprule
Dataset & \multicolumn{2}{c}{vs.\ GAT} & \multicolumn{2}{c}{vs.\ GATv2} \\
\cmidrule(lr){2-3} \cmidrule(lr){4-5}
 & $\Delta$acc & $p$ & $\Delta$acc & $p$ \\
\midrule
Amz-Rat & $+0.8$ & $0.002$$^{**}$ & $+1.1$ & $0.002$$^{**}$ \\
CiteS. & $-0.1$ & $0.748$ & $-0.0$ & $0.935$ \\
Cora & $+0.8$ & $0.146$ & $+0.5$ & $0.124$ \\
Cornell & $-3.0$ & $0.518$ & $+2.4$ & $0.438$ \\
PubMed & $-0.2$ & $0.367$ & $-0.1$ & $0.634$ \\
Rom-Emp & $+16.3$ & $0.000$$^{***}$ & $+1.4$ & $0.006$ \\
Texas & $-0.5$ & $0.780$ & $-0.5$ & $0.693$ \\
Wisc. & $+2.2$ & $0.432$ & $+4.7$ & $0.122$ \\
\bottomrule
\end{tabular}

\end{table}

\begin{table}[t]
\centering
\caption{Per-dataset paired comparisons against LTGA-Edge over $10$
seeds, Holm--Bonferroni corrected, with Cohen's $d$.  This is the test the
paper's claim rests on; the critical-difference diagram
(Figure~\ref{fig:cd_diagram}, Appendix) is the weaker omnibus summary.}
\label{tab:pairwise}
\resizebox{\textwidth}{!}{
\begin{tabular}{lccccc}
\toprule
vs.\ LTGA-Edge & $\Delta$acc (pp) & W/L/T & test & $p$ (Holm) & Cohen's $d$ \\
\midrule
LINKX & $+10.5$ & 5/3/0 & paired $t$ & $0.255$ & $0.44$ \\
GCN & $+6.0$ & 6/2/0 & Wilcoxon & $0.148$ & $0.57$ \\
GAT & $+2.0$ & 4/4/0 & Wilcoxon & $0.547$ & $0.34$ \\
LTGA-Global & $+1.3$ & 7/1/0 & Wilcoxon & $0.039$ & $0.64$ \\
LTGA-Head & $+1.3$ & 6/2/0 & Wilcoxon & $0.055$ & $0.62$ \\
LTGA-Layer & $+1.3$ & 6/2/0 & Wilcoxon & $0.055$ & $0.61$ \\
GATv2 & $+1.2$ & 5/3/0 & paired $t$ & $0.090$ & $0.70$ \\
GATv2-Temp & $+1.2$ & 6/2/0 & Wilcoxon & $0.195$ & $0.53$ \\
FAGCN & $+0.5$ & 4/4/0 & paired $t$ & $0.694$ & $0.14$ \\
GATv2-EdgeGate & $+0.4$ & 2/0/6 & Wilcoxon & $0.500$ & $0.53$ \\
Fixed-$q$ (tuned) & $+0.3$ & 6/2/0 & paired $t$ & $0.124$ & $0.62$ \\
GATv2-EdgeScale & $-0.3$ & 0/2/6 & Wilcoxon & $0.500$ & $-0.54$ \\
$\alpha$-entmax (tuned) & $-0.5$ & 5/3/0 & paired $t$ & $0.621$ & $-0.18$ \\
GPR-GNN & $-1.5$ & 1/7/0 & paired $t$ & $0.211$ & $-0.49$ \\
H2GCN & $-8.7$ & 3/5/0 & paired $t$ & $0.082$ & $-0.72$ \\
\bottomrule
\end{tabular}
}
\end{table}

\begin{table}[t]
\centering
\caption{\textbf{Composition with a different scoring function.}
Rows within a scoring function differ only in the normalisation step, so the
contrast isolates the entropic index from the scoring function it has so far
been bundled with.}
\label{tab:composition}
\resizebox{\textwidth}{!}{
\begin{tabular}{lccccccccc}
\toprule
Method & Amz-Rat & CiteS. & Cora & Cornell & PubMed & Rom-Emp & Texas & Wisc. & Avg. \\
\midrule
Dot-product scoring $+$ softmax & $40.8_{\scriptscriptstyle\pm0.3}$ & $68.7_{\scriptscriptstyle\pm0.5}$ & $\mathbf{81.4}_{\scriptscriptstyle\pm0.6}$ & $\mathbf{42.7}_{\scriptscriptstyle\pm6.7}$ & $\mathbf{77.7}_{\scriptscriptstyle\pm0.6}$ & $62.8_{\scriptscriptstyle\pm0.7}$ & $56.8_{\scriptscriptstyle\pm5.4}$ & $\mathbf{51.2}_{\scriptscriptstyle\pm8.0}$ & $60.3$ \\
Dot-product scoring $+$ learned $q$ & $40.8_{\scriptscriptstyle\pm0.4}$ & $68.7_{\scriptscriptstyle\pm0.5}$ & $81.4_{\scriptscriptstyle\pm0.6}$ & $42.7_{\scriptscriptstyle\pm6.7}$ & $77.7_{\scriptscriptstyle\pm0.6}$ & $65.3_{\scriptscriptstyle\pm0.8}$ & $56.8_{\scriptscriptstyle\pm5.4}$ & $51.2_{\scriptscriptstyle\pm8.0}$ & $60.6$ \\
GATv2 scoring $+$ softmax & $41.7_{\scriptscriptstyle\pm0.7}$ & $\mathbf{69.2}_{\scriptscriptstyle\pm0.7}$ & $80.7_{\scriptscriptstyle\pm0.7}$ & $36.8_{\scriptscriptstyle\pm8.8}$ & $77.6_{\scriptscriptstyle\pm0.7}$ & $71.0_{\scriptscriptstyle\pm0.8}$ & $\mathbf{59.2}_{\scriptscriptstyle\pm5.5}$ & $48.0_{\scriptscriptstyle\pm6.6}$ & $60.5$ \\
GATv2 scoring $+$ learned $q$ & $\mathbf{42.5}_{\scriptscriptstyle\pm0.7}$ & $69.0_{\scriptscriptstyle\pm0.6}$ & $81.0_{\scriptscriptstyle\pm0.7}$ & $35.1_{\scriptscriptstyle\pm12.9}$ & $77.7_{\scriptscriptstyle\pm0.6}$ & $\mathbf{72.5}_{\scriptscriptstyle\pm0.6}$ & $58.1_{\scriptscriptstyle\pm5.9}$ & $47.6_{\scriptscriptstyle\pm7.5}$ & $60.4$ \\
\bottomrule
\end{tabular}
}
\end{table}

\subsection{Is it the family, the learning, or the capacity?}\label{sec:baseline_results}

\paragraph{Frozen versus learned $q$.}
Table~\ref{tab:fixed_q} sweeps a \emph{frozen} index across the family.  The
\emph{tuned} row --- best $q$ per dataset on validation, i.e.\ what a grid
search gives a practitioner --- reaches $61.4\%$, against $60.4\%$ for a
learned scalar $q$ and $61.7\%$ for the per-edge gate.  Learning $q$ therefore
\emph{trails} tuning at scalar granularity and passes it only through
LTGA-Edge, by $0.3$\,pts at $p\!=\!0.124$, while costing one run instead of
$|Q|\!\times\!|D|$.  We report this as a negative result for the central
claim: here the case for learning $q$ is convenience, not accuracy.  The
test-selected oracle in that table is not a baseline; it only bounds what a
perfect per-dataset choice could gain.

\paragraph{Why $q$ stays at $1$ on the homophilic graphs.}
Two different causes, which Figure~\ref{fig:fixed_q} separates.  On Cora,
CiteSeer and PubMed accuracy \emph{decreases} monotonically in $q$
($81.0\!\rightarrow\!80.2$ on Cora) and validation agrees, so softmax is
near-optimal and a learned $q\!=\!1$ is correct.  On Cornell and Wisconsin the
frozen optimum is at $q\!=\!2$ ($39.5$ vs $35.1$; $51.6$ vs $47.6$) yet the
learned index stays at $1.000$ --- here the schedule binds, not the objective.
Figure~\ref{fig:alpha_grad} shows why: Roman-Empire trains the full $200$
epochs, accumulating $\sim\!180$ post-warm-up steps on $\alpha$, whereas early
stopping ends the WebKB graphs after $23$--$27$ epochs, i.e.\ $3$--$7$ steps
past $T_w\!=\!20$.  An index cannot move in that many steps, so this is an
optimisation limitation of our schedule and we report it as one.

\paragraph{Against true $\alpha$-entmax.}
Table~\ref{tab:entmax} compares LTGA with segment-wise $\alpha$-entmax, fixed
and tuned.  Two findings, neither in our favour.  On accuracy, tuned
$\alpha$-entmax is the strongest model in the comparison at $62.2\%$, ahead of
LTGA-Edge ($61.7\%$), with sparsemax at $61.8\%$.  On mechanism, although the
maps are not pointwise equal (Appendix~\ref{app:entmax_relation}),
\emph{trained end to end they reach almost the same sparsity}: on Roman-Empire
over ten seeds $\alpha$-entmax zeroes $42.7\%$ of coefficients against
$43.1\%$ for $q$-softmax at $q\!=\!2$, and they track at $\alpha\!=\!1.2$
($12.0$ vs $13.1$) and $\alpha\!=\!1.5$ ($32.8$ vs $33.6$).  That is what
Proposition~\ref{prop:entmax_rescale} predicts --- the learnable logit scale
absorbs the per-neighbourhood rescaling separating the two maps --- so the
pointwise distinction is real but does not survive training as an empirical
difference, and we no longer claim it as an advantage.

\paragraph{Does the index transfer to another scoring function?}
Every result so far pairs the index with GATv2 scoring, confounding the two.
Table~\ref{tab:composition} separates them by moving the normalisation onto
scaled dot-product attention, $e_{ij} = \langle \mathbf{W}_Q\mathbf{z}_i,
\mathbf{W}_K\mathbf{z}_j\rangle/\sqrt{F_h}$, holding everything else fixed.
(This is the only composition available here; the heterophily baselines have
no normalisation step to replace, see Section~\ref{sec:discussion}.)  The index
transfers: six datasets are bit-identical because $q$ stays at $1$, and
Roman-Empire improves $62.8\!\rightarrow\!65.3$ ($+2.5$\,pts), the same
signature as under GATv2 scoring ($71.0\!\to\!72.5$).  Averaged over eight
benchmarks the difference is $+0.3$\,pts ($p\!=\!0.50$, Wilcoxon), confined as
always to the one benchmark where $q$ moves.  The mechanism therefore belongs
to the normalisation rather than to GAT scoring --- though dot-product scoring
is the weaker backbone in absolute terms.

\paragraph{Capacity-matched controls.}
Table~\ref{tab:edge_controls} spends LTGA-Edge's per-edge MLP on the attention
\emph{logit} instead of on $q$, at exactly matched parameter counts
($1{,}528{,}255$ for LTGA-Edge and all three edge controls; $1{,}526{,}975$
for LTGA-Head and the temperature control).  If an edge-gated GATv2 recovered
LTGA-Edge's accuracy, the per-edge result would be about capacity rather than
entropic geometry, and we would say so.  It does: a per-edge logit
\emph{scale} with $q\!\equiv\!1$ reaches $62.0\%$ against $61.7\%$, winning on
both benchmarks where the gate is active (Amazon-Ratings $44.0$ vs $42.8$,
Roman-Empire $73.8$ vs $72.4$) and tying exactly on the six where it is not; a
learned per-head temperature likewise edges out LTGA-Head ($60.6$ vs $60.4$).
We therefore withdraw the mechanistic reading of LTGA-Edge: its advantage over
GATv2 is per-edge capacity, not entropic geometry.  The temperature control is
the sharpest of the three conceptually --- it re-tempers softmax but cannot
produce an exact zero, which is the one thing $q\!>\!1$ adds.

\subsection{Learned \texorpdfstring{$q$}{q} across the homophily spectrum}\label{sec:q_analysis}

Figure~\ref{fig:q_overview}(a) plots the mean learned $q$ at convergence against the edge homophily $h(\mathcal{G})$ for each dataset.  (An earlier version reported this quantity by averaging all four granularities together, including LTGA-Edge, whose gate has no scalar $q$; that mixed a gate-weight magnitude into the mean and made this figure and Table~\ref{tab:q_analysis} disagree with Table~\ref{tab:ltga_q_stats}.  Both now report a single granularity, and the aggregation is stated in each caption.)  $q$ leaves the Shannon baseline meaningfully on the two large heterophilic benchmarks: \textbf{Roman-Empire} ($h\!=\!0.05$, $q\!=\!1.804\!\pm\!0.009$, $\mathbf{42.4\%}$ \emph{exactly-zero} attention coefficients) and \textbf{Amazon-Ratings} ($h\!=\!0.38$, $q\!=\!1.436\!\pm\!0.055$, $1.2\%$ zero attention).  On the remaining six datasets --- including all three citation networks and all three webKB graphs --- $q$ stays within $10^{-3}$ of $1$ at this budget, so the model is operating in its GAT-equivalent regime.  We do not therefore claim a confirmed monotonic $q$--homophily relationship across all eight benchmarks; with longer post-warm-up training and per-dataset $\lambda_{\mathrm{attn}}$ tuning we expect $q$ to be free to move on more datasets, and we leave that empirical question open.  The companion attention-sparsity panel (Figure~\ref{fig:q_overview}(b)) confirms that, where $q$ moves, $q\!>\!1$ co-occurs with a substantial fraction of exact-zero attention as predicted by Proposition~\ref{prop:dichotomy} part (ii) --- on Roman-Empire nearly half of all attention coefficients are pruned to zero, demonstrating that the network is genuinely exploiting the compact-support regime rather than merely re-tempering softmax.

\begin{figure}[t]
\centering
\includegraphics[width=\linewidth]{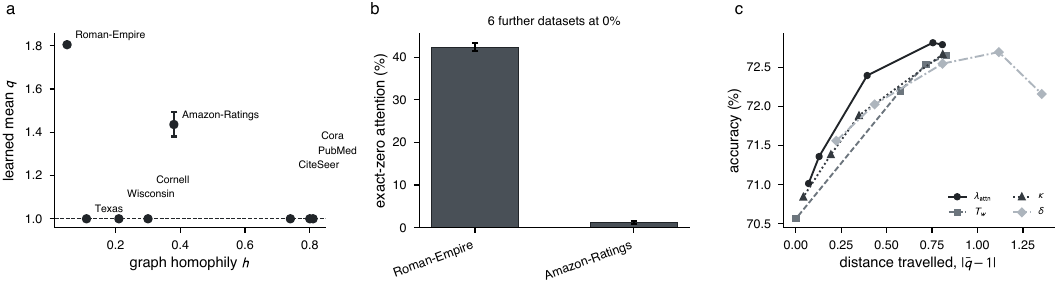}
\caption{\textbf{Learned entropic index, consolidated.}
(a)~Mean learned $q$ vs.\ graph homophily; dashed line is the Shannon
baseline, and labels are stacked so the datasets pinned at $q\!=\!1$ remain
legible.  (b)~Exact-zero attention sparsity, restricted to the datasets where
it is non-zero (the count of datasets at $0\%$ is stated in the panel rather
than drawn as empty bars).  (c)~All four hyperparameter sweeps collapsed onto
one axis, in units of distance travelled from $q\!=\!1$.}
\label{fig:q_overview}
\end{figure}

\begin{table}[t]
\centering
\caption{Pruning validity on Roman-Empire (LTGA-Head).  Label
agreement uses training-split labels only.  ``Random'' prunes the same
fraction of each neighbourhood without regard to the scores.}
\label{tab:pruning}
\begin{tabular}{llcccccc}
\toprule
Dataset & Model & sparsity & \multicolumn{3}{c}{accuracy (\%)} & \multicolumn{2}{c}{label agreement} \\
\cmidrule(lr){4-6} \cmidrule(lr){7-8}
 & & & learned & all edges & random & pruned & kept \\
\midrule
Amazon-Ratings & LTGA-Edge & $2.6\%$ & $42.7_{\scriptscriptstyle\pm0.8}$ & $41.7_{\scriptscriptstyle\pm0.6}$ & $42.7_{\scriptscriptstyle\pm0.8}$  & n/a & $51.0_{\scriptscriptstyle\pm0.3}$ \\
Amazon-Ratings & LTGA-Head & $1.0\%$ & $42.5_{\scriptscriptstyle\pm0.6}$ & $41.8_{\scriptscriptstyle\pm0.5}$ & $42.5_{\scriptscriptstyle\pm0.6}$  & n/a & $51.0_{\scriptscriptstyle\pm0.3}$ \\
Roman-Empire & LTGA-Edge & $41.7\%$ & $72.3_{\scriptscriptstyle\pm0.8}$ & $64.4_{\scriptscriptstyle\pm0.5}$ & $59.8_{\scriptscriptstyle\pm0.6}$ & $2.4_{\scriptscriptstyle\pm2.7}$ & $43.7_{\scriptscriptstyle\pm0.2}$ \\
Roman-Empire & LTGA-Head & $42.4\%$ & $72.4_{\scriptscriptstyle\pm0.8}$ & $65.4_{\scriptscriptstyle\pm0.7}$ & $59.4_{\scriptscriptstyle\pm0.8}$ & $5.2_{\scriptscriptstyle\pm9.0}$ & $43.6_{\scriptscriptstyle\pm0.1}$ \\
\bottomrule
\end{tabular}

\end{table}

\begin{figure}[t]
\centering
\includegraphics[width=\linewidth]{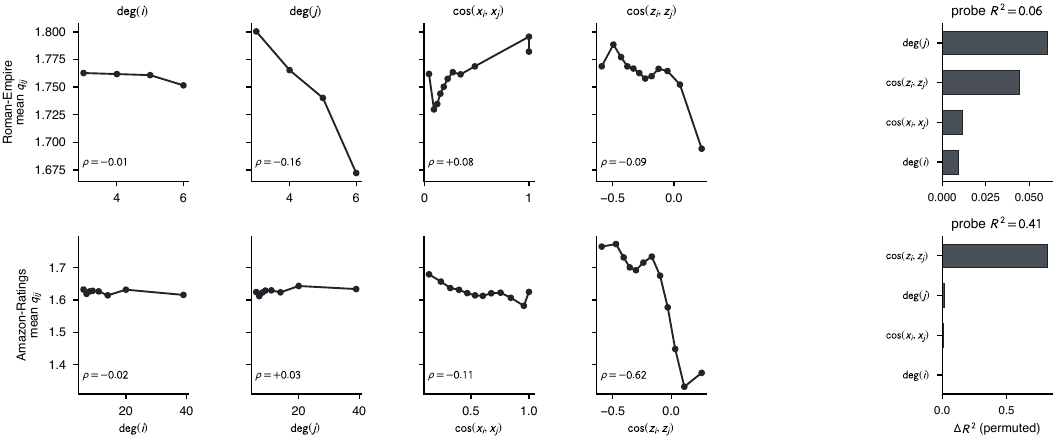}
\caption{\textbf{What the LTGA-Edge gate conditions on.}  One row per
benchmark, sharing a $y$-axis; binned mean $q_{ij}$ against endpoint degrees and
raw-feature / embedding cosine similarity, with Spearman $\rho$ inset.  The
right-hand column gives permutation importances from a linear probe on all four
covariates, and the probe $R^2$.  The gate is informative on Amazon-Ratings
($R^2\!=\!0.41$, driven almost entirely by embedding similarity,
$\rho\!=\!-0.62$) but close to unexplained on Roman-Empire
($R^2\!=\!0.06$), so we do not claim the per-edge index is generally
interpretable.  The remaining six benchmarks are omitted because every
$q_{ij}$ there equals $1$.}
\label{fig:gate_analysis}
\end{figure}

\begin{figure}[t]
\centering
\includegraphics[width=\linewidth]{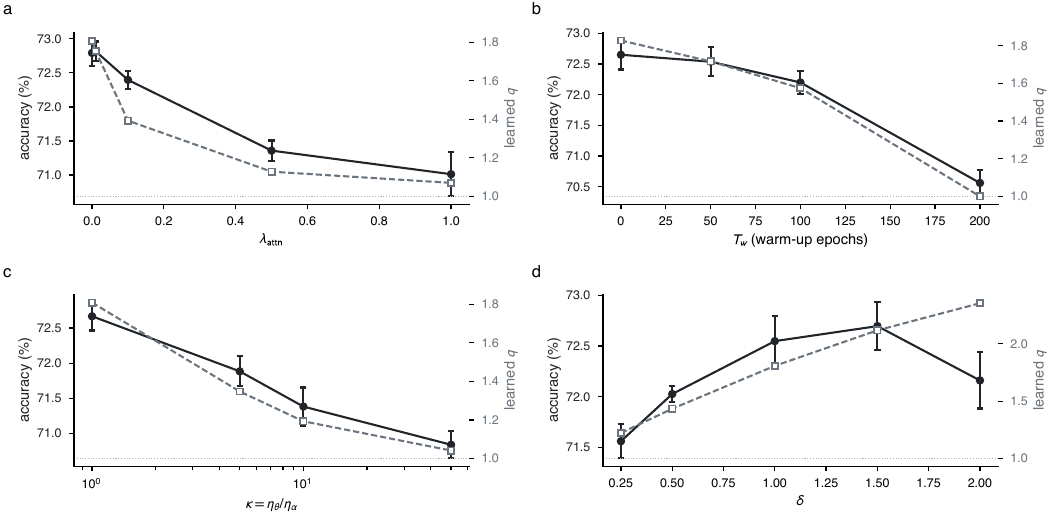}
\caption{Ablation sweeps on Roman-Empire (LTGA-Head, mean $\pm$ s.e.m.\ over
seeds).  Filled circles on a solid line (left axis): test accuracy; open
squares on a dashed line (right axis): learned mean $q$.  The horizontal dotted
line marks $q\!=\!1$.  Panels: (a)~$\lambda_{\mathrm{attn}}$, (b)~$T_w$,
(c)~$\kappa$ (log scale), (d)~$\delta$.}
\label{fig:ablations}
\end{figure}

\section{Per-dataset and ablation tables}\label{app:tables}

This appendix reports the full per-dataset numerics summarised in Section~\ref{sec:results}.  Tables~\ref{tab:ltga_acc}--\ref{tab:ltga_q_stats} compare the four LTGA granularities head-to-head on every benchmark; Table~\ref{tab:q_analysis} reports the full LTGA-Head trajectory; Tables~\ref{tab:abl_lambda}--\ref{tab:abl_granularity} are the four ablation sweeps from Figure~\ref{fig:ablations}.  Figure~\ref{fig:q_traj} shows the per-epoch $q$-trajectories underlying the convergence values.

\begin{table}[h]
\centering
\caption{Per-dataset accuracy (\%) for the four LTGA granularities at $n\!=\!3$ seeds.  Bold marks the best LTGA variant per dataset.  LTGA-Edge takes the most cells (5 of 8) and the highest average accuracy under the granularity sweep, but the four variants overlap within seed variance on six of the eight datasets.}
\label{tab:ltga_acc}
\resizebox{\textwidth}{!}{\begin{tabular}{lccccccccc}
\toprule
Variant & Amz-Rat & CiteS. & Cora & Cornell & PubMed & Rom-Emp & Texas & Wisc. & Avg.\,acc \\
\midrule
LTGA-Global & $42.4_{\scriptscriptstyle\pm0.7}$ & $69.0_{\scriptscriptstyle\pm0.6}$ & $81.0_{\scriptscriptstyle\pm0.7}$ & $35.1_{\scriptscriptstyle\pm12.9}$ & $\mathbf{77.7}_{\scriptscriptstyle\pm0.6}$ & $72.3_{\scriptscriptstyle\pm0.6}$ & $58.1_{\scriptscriptstyle\pm5.9}$ & $47.6_{\scriptscriptstyle\pm7.5}$ & $60.4$ \\
LTGA-Layer & $42.6_{\scriptscriptstyle\pm0.6}$ & $69.0_{\scriptscriptstyle\pm0.6}$ & $81.0_{\scriptscriptstyle\pm0.7}$ & $35.1_{\scriptscriptstyle\pm12.9}$ & $77.7_{\scriptscriptstyle\pm0.6}$ & $72.4_{\scriptscriptstyle\pm0.5}$ & $58.1_{\scriptscriptstyle\pm5.9}$ & $47.6_{\scriptscriptstyle\pm7.5}$ & $60.4$ \\
LTGA-Head & $42.5_{\scriptscriptstyle\pm0.7}$ & $69.0_{\scriptscriptstyle\pm0.6}$ & $81.0_{\scriptscriptstyle\pm0.7}$ & $35.1_{\scriptscriptstyle\pm12.9}$ & $77.7_{\scriptscriptstyle\pm0.6}$ & $\mathbf{72.5}_{\scriptscriptstyle\pm0.6}$ & $58.1_{\scriptscriptstyle\pm5.9}$ & $47.6_{\scriptscriptstyle\pm7.5}$ & $60.4$ \\
LTGA-Edge & $\mathbf{42.8}_{\scriptscriptstyle\pm0.8}$ & $\mathbf{69.1}_{\scriptscriptstyle\pm0.9}$ & $\mathbf{81.2}_{\scriptscriptstyle\pm1.0}$ & $\mathbf{39.2}_{\scriptscriptstyle\pm11.1}$ & $77.5_{\scriptscriptstyle\pm0.6}$ & $72.4_{\scriptscriptstyle\pm0.9}$ & $\mathbf{58.6}_{\scriptscriptstyle\pm6.5}$ & $\mathbf{52.7}_{\scriptscriptstyle\pm6.5}$ & $61.7$ \\
\bottomrule
\end{tabular}
}
\end{table}

\begin{table}[h]
\centering
\caption{Per-dataset learned $\bar q$ (mean across seeds) and exact-zero attention sparsity for each LTGA granularity.  $q$ leaves the Shannon baseline only on the two large heterophilic benchmarks (Amazon-Ratings, Roman-Empire); on Roman-Empire all four granularities converge to $\bar q\!\approx\!1.85$ with $\sim\!44$--$45\%$ exactly-zero attention coefficients.  ``--'' for LTGA-Edge: per-edge gates produce a distribution of $q_{ij}$ rather than a single scalar, so we report sparsity only.}
\label{tab:ltga_q_stats}
\resizebox{\textwidth}{!}{\begin{tabular}{lcccccccc}
\toprule
Dataset & \multicolumn{2}{c}{LTGA-Global} & \multicolumn{2}{c}{LTGA-Layer} & \multicolumn{2}{c}{LTGA-Head} & \multicolumn{2}{c}{LTGA-Edge} \\
\cmidrule(lr){2-3} \cmidrule(lr){4-5} \cmidrule(lr){6-7} \cmidrule(lr){8-9}
 &  $\bar q$ &  spar. &  $\bar q$ &  spar. &  $\bar q$ &  spar. &  $\bar q$ &  spar. \\
\midrule
Amz-Rat & 1.78 & 1.4\% & 1.64 & 1.2\% & 1.44 & 1.2\% & -- & 2.7\% \\
CiteS. & 1.00 & 0.0\% & 1.00 & 0.0\% & 1.00 & 0.0\% & -- & 0.0\% \\
Cora & 1.00 & 0.0\% & 1.00 & 0.0\% & 1.00 & 0.0\% & -- & 0.0\% \\
Cornell & 1.00 & 0.0\% & 1.00 & 0.0\% & 1.00 & 0.0\% & -- & 0.0\% \\
PubMed & 1.00 & 0.0\% & 1.00 & 0.0\% & 1.00 & 0.0\% & -- & 0.0\% \\
Rom-Emp & 1.88 & 43.6\% & 1.85 & 43.2\% & 1.80 & 42.4\% & -- & 41.5\% \\
Texas & 1.00 & 0.0\% & 1.00 & 0.0\% & 1.00 & 0.0\% & -- & 0.0\% \\
Wisc. & 1.00 & 0.0\% & 1.00 & 0.0\% & 1.00 & 0.0\% & -- & 0.0\% \\
\bottomrule
\end{tabular}
}
\end{table}

\begin{figure}[h]
\centering
\includegraphics[width=\linewidth]{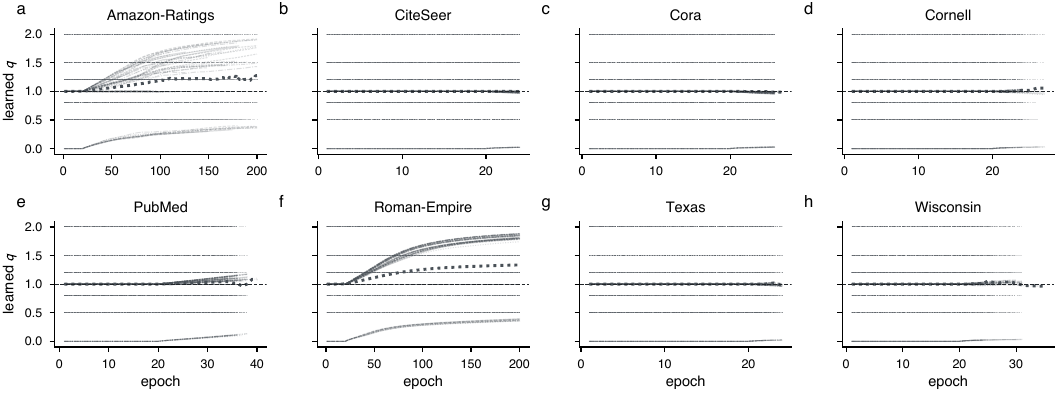}
\caption{Trajectory of the mean learned $q$ over training epochs for each
dataset.  Light traces are individual seeds; bold traces are seed averages.
The Shannon baseline ($q\!=\!1$) is shown as a dashed reference.}
\label{fig:q_traj}
\end{figure}

\begin{table}[h]
\centering
\caption{Per-dataset learned $q$, attention sparsity, and effective
neighbourhood size $1/\!\sum_j \alpha_{ij}^2$ for LTGA-Head.}
\label{tab:q_analysis}
\begin{tabular}{lccccc}
\toprule
Dataset & Homophily & Mean $q$ & $q$ std & Sparsity & Eff. nbrs.\ \\
\midrule
Amazon-Ratings & 0.38 & 1.436 & 0.055 & 0.012 & 8.32 \\
CiteSeer & 0.74 & 1.000 & 0.000 & 0.000 & 3.69 \\
Cora & 0.81 & 1.000 & 0.000 & 0.000 & 4.72 \\
Cornell & 0.30 & 1.000 & 0.000 & 0.000 & 2.41 \\
PubMed & 0.80 & 1.000 & 0.000 & 0.000 & 5.43 \\
Roman-Empire & 0.05 & 1.804 & 0.009 & 0.424 & 1.89 \\
Texas & 0.11 & 1.000 & 0.000 & 0.000 & 2.57 \\
Wisconsin & 0.21 & 1.000 & 0.000 & 0.000 & 2.56 \\
\bottomrule
\end{tabular}

\end{table}


\begin{table}[h]
\centering
\caption{\textbf{Frozen versus learned entropic index.}  Accuracy
(\%) with $q$ held fixed throughout training.  $q\!=\!1.0$ is a bit-exact
GATv2 control; $q\!=\!2.0$ is the compact-support end-point.  ``Tuned''
selects $q$ per dataset on validation accuracy --- the fair fixed-$q$
baseline.  The test-selected oracle is an upper bound only and is not a
baseline any practitioner could use.}
\label{tab:fixed_q}
\resizebox{\textwidth}{!}{
\begin{tabular}{lccccccccc}
\toprule
Method & Amz-Rat & CiteS. & Cora & Cornell & PubMed & Rom-Emp & Texas & Wisc. & Avg. \\
\midrule
$q\!=\!0.5$ (frozen) & $41.1_{\scriptscriptstyle\pm0.3}$ & $69.0_{\scriptscriptstyle\pm0.7}$ & $81.0_{\scriptscriptstyle\pm0.6}$ & $34.6_{\scriptscriptstyle\pm12.2}$ & $77.6_{\scriptscriptstyle\pm0.6}$ & $68.0_{\scriptscriptstyle\pm0.5}$ & $58.4_{\scriptscriptstyle\pm6.0}$ & $48.0_{\scriptscriptstyle\pm7.0}$ & $59.7$ \\
$q\!=\!0.8$ (frozen) & $41.4_{\scriptscriptstyle\pm0.3}$ & $69.0_{\scriptscriptstyle\pm0.7}$ & $81.0_{\scriptscriptstyle\pm0.6}$ & $35.9_{\scriptscriptstyle\pm13.2}$ & $\mathbf{77.7}_{\scriptscriptstyle\pm0.6}$ & $69.8_{\scriptscriptstyle\pm0.5}$ & $58.1_{\scriptscriptstyle\pm5.9}$ & $47.6_{\scriptscriptstyle\pm8.0}$ & $60.1$ \\
$q\!=\!1.0$ (frozen) & $41.9_{\scriptscriptstyle\pm0.6}$ & $69.0_{\scriptscriptstyle\pm0.6}$ & $81.0_{\scriptscriptstyle\pm0.7}$ & $35.1_{\scriptscriptstyle\pm12.9}$ & $77.7_{\scriptscriptstyle\pm0.6}$ & $70.7_{\scriptscriptstyle\pm0.6}$ & $58.1_{\scriptscriptstyle\pm5.9}$ & $47.6_{\scriptscriptstyle\pm7.5}$ & $60.1$ \\
$q\!=\!1.2$ (frozen) & $42.2_{\scriptscriptstyle\pm0.5}$ & $69.0_{\scriptscriptstyle\pm0.7}$ & $80.9_{\scriptscriptstyle\pm0.6}$ & $35.7_{\scriptscriptstyle\pm12.9}$ & $77.6_{\scriptscriptstyle\pm0.7}$ & $71.5_{\scriptscriptstyle\pm0.7}$ & $58.1_{\scriptscriptstyle\pm5.9}$ & $47.6_{\scriptscriptstyle\pm7.3}$ & $60.3$ \\
$q\!=\!1.5$ (frozen) & $42.5_{\scriptscriptstyle\pm0.5}$ & $68.9_{\scriptscriptstyle\pm0.7}$ & $80.7_{\scriptscriptstyle\pm0.7}$ & $36.2_{\scriptscriptstyle\pm12.5}$ & $77.5_{\scriptscriptstyle\pm0.6}$ & $72.0_{\scriptscriptstyle\pm0.5}$ & $57.6_{\scriptscriptstyle\pm6.0}$ & $50.0_{\scriptscriptstyle\pm10.1}$ & $60.7$ \\
$q\!=\!2.0$ (frozen) & $42.8_{\scriptscriptstyle\pm0.7}$ & $68.9_{\scriptscriptstyle\pm0.7}$ & $80.2_{\scriptscriptstyle\pm0.7}$ & $\mathbf{39.5}_{\scriptscriptstyle\pm9.6}$ & $77.4_{\scriptscriptstyle\pm0.6}$ & $71.9_{\scriptscriptstyle\pm0.5}$ & $56.2_{\scriptscriptstyle\pm8.8}$ & $51.6_{\scriptscriptstyle\pm6.3}$ & $61.0$ \\
LTGA-Head (learned $q$) & $42.5_{\scriptscriptstyle\pm0.7}$ & $69.0_{\scriptscriptstyle\pm0.6}$ & $81.0_{\scriptscriptstyle\pm0.7}$ & $35.1_{\scriptscriptstyle\pm12.9}$ & $77.7_{\scriptscriptstyle\pm0.6}$ & $\mathbf{72.5}_{\scriptscriptstyle\pm0.6}$ & $58.1_{\scriptscriptstyle\pm5.9}$ & $47.6_{\scriptscriptstyle\pm7.5}$ & $60.4$ \\
LTGA-Edge (learned $q_{ij}$) & $\mathbf{42.8}_{\scriptscriptstyle\pm0.8}$ & $\mathbf{69.1}_{\scriptscriptstyle\pm0.9}$ & $\mathbf{81.2}_{\scriptscriptstyle\pm1.0}$ & $39.2_{\scriptscriptstyle\pm11.1}$ & $77.5_{\scriptscriptstyle\pm0.6}$ & $72.4_{\scriptscriptstyle\pm0.9}$ & $\mathbf{58.6}_{\scriptscriptstyle\pm6.5}$ & $\mathbf{52.7}_{\scriptscriptstyle\pm6.5}$ & $61.7$ \\
\midrule
Fixed-$q$ (tuned on val.) & $42.8_{\scriptscriptstyle\pm0.7}$ & $69.0_{\scriptscriptstyle\pm0.7}$ & $81.0_{\scriptscriptstyle\pm0.6}$ & $39.5_{\scriptscriptstyle\pm9.6}$ & $77.7_{\scriptscriptstyle\pm0.6}$ & $72.0_{\scriptscriptstyle\pm0.5}$ & $57.6_{\scriptscriptstyle\pm6.0}$ & $51.6_{\scriptscriptstyle\pm6.3}$ & $61.4$ \\
\midrule
Fixed-$q$ (test oracle; \emph{not} a baseline) & $42.8_{\scriptscriptstyle\pm0.7}$ & $69.0_{\scriptscriptstyle\pm0.7}$ & $81.0_{\scriptscriptstyle\pm0.6}$ & $39.5_{\scriptscriptstyle\pm9.6}$ & $77.7_{\scriptscriptstyle\pm0.6}$ & $72.0_{\scriptscriptstyle\pm0.5}$ & $58.4_{\scriptscriptstyle\pm6.0}$ & $51.6_{\scriptscriptstyle\pm6.3}$ & $61.5$ \\
\bottomrule
\end{tabular}
}
\end{table}

\begin{figure}[h]
\centering
\includegraphics[width=\linewidth]{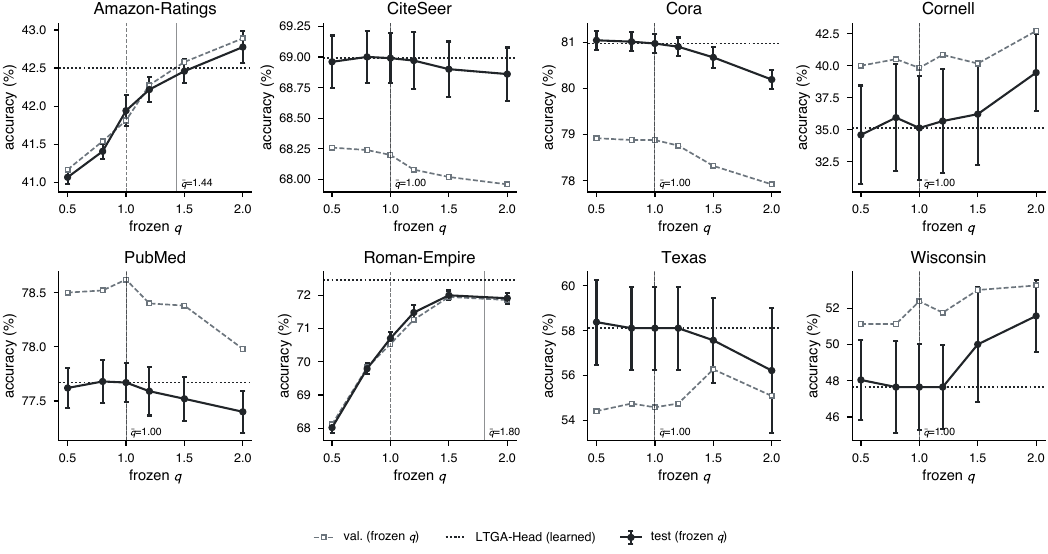}
\caption{\textbf{Accuracy against a frozen $q$, per dataset.}  Solid:
test; dashed: validation; dotted horizontal line: learned-$q$ LTGA-Head.  A
flat curve means softmax is genuinely near-optimal on that graph, so a learned
$q$ that stays at $1$ is the right answer rather than an optimisation failure.}
\label{fig:fixed_q}
\end{figure}

\begin{figure}[h]
\centering
\includegraphics[width=0.6\linewidth]{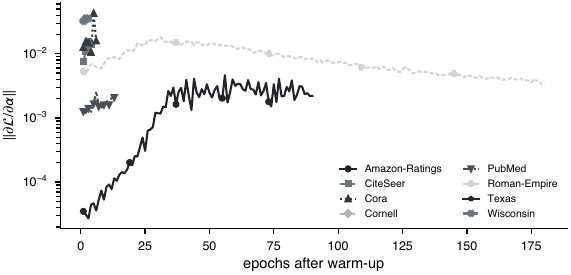}
\caption{\textbf{Gradient signal reaching the entropic index.}  Norm
of $\partial\mathcal{L}/\partial\alpha$ per epoch after the warm-up, log
scale.  A vanishing norm on the homophilic graphs indicates there is nothing to
gain from moving $q$; a large norm with $q$ still pinned would instead indicate
that the schedule is suppressing it.}
\label{fig:alpha_grad}
\end{figure}

\begin{table}[h]
\centering
\caption{\textbf{LTGA versus true $\alpha$-entmax attention.}
Segment-wise $\alpha$-entmax with the exact entmax Jacobian, at fixed and
per-dataset-tuned $\alpha$, against frozen and learned $q$-softmax.  The
$q\!=\!\alpha$ rows are the like-for-like comparison; they are close but not
identical maps (Appendix~\ref{app:entmax_relation}).}
\label{tab:entmax}
\resizebox{\textwidth}{!}{
\begin{tabular}{lccccccccc}
\toprule
Method & Amz-Rat & CiteS. & Cora & Cornell & PubMed & Rom-Emp & Texas & Wisc. & Avg. \\
\midrule
$\alpha$-entmax, $\alpha\!=\!1.2$ & $42.7_{\scriptscriptstyle\pm0.6}$ & $68.9_{\scriptscriptstyle\pm0.8}$ & $80.8_{\scriptscriptstyle\pm0.7}$ & $35.9_{\scriptscriptstyle\pm13.4}$ & $77.5_{\scriptscriptstyle\pm0.8}$ & $71.4_{\scriptscriptstyle\pm0.8}$ & $58.1_{\scriptscriptstyle\pm5.9}$ & $48.6_{\scriptscriptstyle\pm6.2}$ & $60.5$ \\
$\alpha$-entmax, $\alpha\!=\!1.5$ & $43.9_{\scriptscriptstyle\pm0.5}$ & $68.6_{\scriptscriptstyle\pm0.9}$ & $80.1_{\scriptscriptstyle\pm0.6}$ & $39.5_{\scriptscriptstyle\pm11.6}$ & $77.3_{\scriptscriptstyle\pm0.7}$ & $72.3_{\scriptscriptstyle\pm0.6}$ & $56.2_{\scriptscriptstyle\pm9.3}$ & $48.4_{\scriptscriptstyle\pm9.1}$ & $60.8$ \\
sparsemax ($\alpha\!=\!2$) & $\mathbf{44.7}_{\scriptscriptstyle\pm0.5}$ & $68.2_{\scriptscriptstyle\pm0.9}$ & $78.9_{\scriptscriptstyle\pm0.5}$ & $\mathbf{45.4}_{\scriptscriptstyle\pm5.9}$ & $77.0_{\scriptscriptstyle\pm0.7}$ & $\mathbf{73.1}_{\scriptscriptstyle\pm0.5}$ & $57.3_{\scriptscriptstyle\pm5.8}$ & $49.6_{\scriptscriptstyle\pm7.7}$ & $61.8$ \\
$q$-softmax, $q\!=\!1.5$ (frozen) & $42.5_{\scriptscriptstyle\pm0.5}$ & $68.9_{\scriptscriptstyle\pm0.7}$ & $80.7_{\scriptscriptstyle\pm0.7}$ & $36.2_{\scriptscriptstyle\pm12.5}$ & $77.5_{\scriptscriptstyle\pm0.6}$ & $72.0_{\scriptscriptstyle\pm0.5}$ & $57.6_{\scriptscriptstyle\pm6.0}$ & $50.0_{\scriptscriptstyle\pm10.1}$ & $60.7$ \\
$q$-softmax, $q\!=\!2.0$ (frozen) & $42.8_{\scriptscriptstyle\pm0.7}$ & $68.9_{\scriptscriptstyle\pm0.7}$ & $80.2_{\scriptscriptstyle\pm0.7}$ & $39.5_{\scriptscriptstyle\pm9.6}$ & $77.4_{\scriptscriptstyle\pm0.6}$ & $71.9_{\scriptscriptstyle\pm0.5}$ & $56.2_{\scriptscriptstyle\pm8.8}$ & $51.6_{\scriptscriptstyle\pm6.3}$ & $61.0$ \\
LTGA-Head (learned $q$) & $42.5_{\scriptscriptstyle\pm0.7}$ & $69.0_{\scriptscriptstyle\pm0.6}$ & $81.0_{\scriptscriptstyle\pm0.7}$ & $35.1_{\scriptscriptstyle\pm12.9}$ & $\mathbf{77.7}_{\scriptscriptstyle\pm0.6}$ & $72.5_{\scriptscriptstyle\pm0.6}$ & $58.1_{\scriptscriptstyle\pm5.9}$ & $47.6_{\scriptscriptstyle\pm7.5}$ & $60.4$ \\
LTGA-Edge (learned $q_{ij}$) & $42.8_{\scriptscriptstyle\pm0.8}$ & $\mathbf{69.1}_{\scriptscriptstyle\pm0.9}$ & $\mathbf{81.2}_{\scriptscriptstyle\pm1.0}$ & $39.2_{\scriptscriptstyle\pm11.1}$ & $77.5_{\scriptscriptstyle\pm0.6}$ & $72.4_{\scriptscriptstyle\pm0.9}$ & $\mathbf{58.6}_{\scriptscriptstyle\pm6.5}$ & $\mathbf{52.7}_{\scriptscriptstyle\pm6.5}$ & $61.7$ \\
\midrule
$\alpha$-entmax (tuned on val.) & $44.7_{\scriptscriptstyle\pm0.5}$ & $68.9_{\scriptscriptstyle\pm0.8}$ & $80.8_{\scriptscriptstyle\pm0.7}$ & $45.4_{\scriptscriptstyle\pm5.9}$ & $77.0_{\scriptscriptstyle\pm0.7}$ & $73.1_{\scriptscriptstyle\pm0.5}$ & $58.1_{\scriptscriptstyle\pm5.9}$ & $49.6_{\scriptscriptstyle\pm7.7}$ & $62.2$ \\
\bottomrule
\end{tabular}
}
\end{table}

\begin{table}[h]
\centering
\caption{\textbf{Capacity-matched controls for LTGA-Edge.}  Each
control reuses LTGA-Edge's per-edge MLP but spends it on the attention logit
(additive bias or multiplicative scale) or replaces it with a learned per-head
temperature, keeping $q\!\equiv\!1$.  Trainable parameters match exactly:
$1{,}528{,}255$ for LTGA-Edge and both edge controls, $1{,}526{,}975$ for
LTGA-Head and the temperature control, $1{,}526{,}959$ for GATv2 (Cora).}
\label{tab:edge_controls}
\resizebox{\textwidth}{!}{
\begin{tabular}{lccccccccc}
\toprule
Method & Amz-Rat & CiteS. & Cora & Cornell & PubMed & Rom-Emp & Texas & Wisc. & Avg. \\
\midrule
+ per-edge logit bias ($q\!=\!1$) & $41.6_{\scriptscriptstyle\pm0.7}$ & $69.1_{\scriptscriptstyle\pm0.9}$ & $\mathbf{81.2}_{\scriptscriptstyle\pm1.0}$ & $\mathbf{39.2}_{\scriptscriptstyle\pm11.1}$ & $77.5_{\scriptscriptstyle\pm0.6}$ & $70.8_{\scriptscriptstyle\pm0.8}$ & $58.6_{\scriptscriptstyle\pm6.5}$ & $\mathbf{52.7}_{\scriptscriptstyle\pm6.5}$ & $61.4$ \\
+ per-edge logit scale ($q\!=\!1$) & $\mathbf{44.0}_{\scriptscriptstyle\pm0.6}$ & $69.1_{\scriptscriptstyle\pm0.9}$ & $81.2_{\scriptscriptstyle\pm1.0}$ & $39.2_{\scriptscriptstyle\pm11.1}$ & $77.5_{\scriptscriptstyle\pm0.6}$ & $\mathbf{73.8}_{\scriptscriptstyle\pm0.8}$ & $58.6_{\scriptscriptstyle\pm6.5}$ & $52.7_{\scriptscriptstyle\pm6.5}$ & $62.0$ \\
+ learned per-head temperature & $42.7_{\scriptscriptstyle\pm0.6}$ & $69.0_{\scriptscriptstyle\pm0.6}$ & $81.0_{\scriptscriptstyle\pm0.7}$ & $35.1_{\scriptscriptstyle\pm12.9}$ & $\mathbf{77.7}_{\scriptscriptstyle\pm0.6}$ & $73.3_{\scriptscriptstyle\pm0.6}$ & $58.1_{\scriptscriptstyle\pm5.9}$ & $47.6_{\scriptscriptstyle\pm7.5}$ & $60.6$ \\
GATv2 (softmax) & $41.7_{\scriptscriptstyle\pm0.7}$ & $\mathbf{69.2}_{\scriptscriptstyle\pm0.7}$ & $80.7_{\scriptscriptstyle\pm0.7}$ & $36.8_{\scriptscriptstyle\pm8.8}$ & $77.6_{\scriptscriptstyle\pm0.7}$ & $71.0_{\scriptscriptstyle\pm0.8}$ & $\mathbf{59.2}_{\scriptscriptstyle\pm5.5}$ & $48.0_{\scriptscriptstyle\pm6.6}$ & $60.5$ \\
LTGA-Head (learned $q$) & $42.5_{\scriptscriptstyle\pm0.7}$ & $69.0_{\scriptscriptstyle\pm0.6}$ & $81.0_{\scriptscriptstyle\pm0.7}$ & $35.1_{\scriptscriptstyle\pm12.9}$ & $77.7_{\scriptscriptstyle\pm0.6}$ & $72.5_{\scriptscriptstyle\pm0.6}$ & $58.1_{\scriptscriptstyle\pm5.9}$ & $47.6_{\scriptscriptstyle\pm7.5}$ & $60.4$ \\
LTGA-Edge (learned $q_{ij}$) & $42.8_{\scriptscriptstyle\pm0.8}$ & $69.1_{\scriptscriptstyle\pm0.9}$ & $81.2_{\scriptscriptstyle\pm1.0}$ & $39.2_{\scriptscriptstyle\pm11.1}$ & $77.5_{\scriptscriptstyle\pm0.6}$ & $72.4_{\scriptscriptstyle\pm0.9}$ & $58.6_{\scriptscriptstyle\pm6.5}$ & $52.7_{\scriptscriptstyle\pm6.5}$ & $61.7$ \\
\bottomrule
\end{tabular}
%
}
\end{table}

\begin{table}[h]
\centering
\caption{\textbf{Heterophily-specific architectures.}  These change
propagation rather than normalisation and are reported as context; the gap to
published graph-transformer results on Roman-Empire is discussed in
Section~\ref{sec:discussion} rather than elided.}
\label{tab:hetero}
\resizebox{\textwidth}{!}{
\begin{tabular}{lccccccccc}
\toprule
Method & Amz-Rat & CiteS. & Cora & Cornell & PubMed & Rom-Emp & Texas & Wisc. & Avg. \\
\midrule
H2GCN & $\mathbf{45.3}_{\scriptscriptstyle\pm0.7}$ & $66.8_{\scriptscriptstyle\pm1.9}$ & $80.4_{\scriptscriptstyle\pm0.9}$ & $\mathbf{65.1}_{\scriptscriptstyle\pm5.3}$ & $73.8_{\scriptscriptstyle\pm1.0}$ & $\mathbf{78.3}_{\scriptscriptstyle\pm0.7}$ & $\mathbf{77.0}_{\scriptscriptstyle\pm7.5}$ & $76.5_{\scriptscriptstyle\pm6.5}$ & $70.4$ \\
GPR-GNN & $44.0_{\scriptscriptstyle\pm0.7}$ & $\mathbf{70.1}_{\scriptscriptstyle\pm0.4}$ & $\mathbf{81.8}_{\scriptscriptstyle\pm0.7}$ & $42.7_{\scriptscriptstyle\pm4.4}$ & $\mathbf{77.9}_{\scriptscriptstyle\pm0.8}$ & $72.6_{\scriptscriptstyle\pm0.5}$ & $55.9_{\scriptscriptstyle\pm6.2}$ & $60.4_{\scriptscriptstyle\pm12.6}$ & $63.2$ \\
FAGCN & $42.2_{\scriptscriptstyle\pm0.5}$ & $66.2_{\scriptscriptstyle\pm1.8}$ & $80.1_{\scriptscriptstyle\pm0.9}$ & $43.2_{\scriptscriptstyle\pm8.1}$ & $77.6_{\scriptscriptstyle\pm0.6}$ & $64.8_{\scriptscriptstyle\pm1.5}$ & $60.3_{\scriptscriptstyle\pm5.8}$ & $55.1_{\scriptscriptstyle\pm5.0}$ & $61.2$ \\
LINKX & $44.7_{\scriptscriptstyle\pm2.1}$ & $40.2_{\scriptscriptstyle\pm7.0}$ & $44.3_{\scriptscriptstyle\pm7.2}$ & $51.1_{\scriptscriptstyle\pm24.8}$ & $39.4_{\scriptscriptstyle\pm13.6}$ & $54.3_{\scriptscriptstyle\pm2.5}$ & $56.2_{\scriptscriptstyle\pm27.8}$ & $\mathbf{79.8}_{\scriptscriptstyle\pm4.1}$ & $51.3$ \\
GAT & $42.0_{\scriptscriptstyle\pm0.4}$ & $69.3_{\scriptscriptstyle\pm1.0}$ & $80.4_{\scriptscriptstyle\pm0.8}$ & $42.2_{\scriptscriptstyle\pm7.3}$ & $77.7_{\scriptscriptstyle\pm0.5}$ & $56.1_{\scriptscriptstyle\pm0.7}$ & $59.2_{\scriptscriptstyle\pm6.2}$ & $50.6_{\scriptscriptstyle\pm7.3}$ & $59.7$ \\
GATv2 & $41.7_{\scriptscriptstyle\pm0.7}$ & $69.2_{\scriptscriptstyle\pm0.7}$ & $80.7_{\scriptscriptstyle\pm0.7}$ & $36.8_{\scriptscriptstyle\pm8.8}$ & $77.6_{\scriptscriptstyle\pm0.7}$ & $71.0_{\scriptscriptstyle\pm0.8}$ & $59.2_{\scriptscriptstyle\pm5.5}$ & $48.0_{\scriptscriptstyle\pm6.6}$ & $60.5$ \\
LTGA-Head & $42.5_{\scriptscriptstyle\pm0.7}$ & $69.0_{\scriptscriptstyle\pm0.6}$ & $81.0_{\scriptscriptstyle\pm0.7}$ & $35.1_{\scriptscriptstyle\pm12.9}$ & $77.7_{\scriptscriptstyle\pm0.6}$ & $72.5_{\scriptscriptstyle\pm0.6}$ & $58.1_{\scriptscriptstyle\pm5.9}$ & $47.6_{\scriptscriptstyle\pm7.5}$ & $60.4$ \\
LTGA-Edge & $42.8_{\scriptscriptstyle\pm0.8}$ & $69.1_{\scriptscriptstyle\pm0.9}$ & $81.2_{\scriptscriptstyle\pm1.0}$ & $39.2_{\scriptscriptstyle\pm11.1}$ & $77.5_{\scriptscriptstyle\pm0.6}$ & $72.4_{\scriptscriptstyle\pm0.9}$ & $58.6_{\scriptscriptstyle\pm6.5}$ & $52.7_{\scriptscriptstyle\pm6.5}$ & $61.7$ \\
\bottomrule
\end{tabular}
}
\end{table}

\begin{table}[h]
\centering
\caption{\textbf{Cross-dataset hyperparameter ablation.}  The
$\lambda_{\mathrm{attn}}\!\times\!T_w\!\times\!\kappa$ grid repeated on a
homophilic (Cora), a mid-homophily (Amazon-Ratings) and a low-homophily
(Wisconsin) graph, addressing the circularity of sweeping only on
Roman-Empire.}
\label{tab:abl_cross}
{\small
\begin{tabular}{llccc}
\toprule
Dataset & Parameter & Value & Accuracy (\%) & Converged $\bar q$ \\
\midrule
Amazon-Ratings & $\lambda_{\mathrm{attn}}$ & 0.0 & $42.9_{\scriptscriptstyle\pm0.3}$ & $1.452$ \\
Amazon-Ratings & $\lambda_{\mathrm{attn}}$ & 0.01 & $42.2_{\scriptscriptstyle\pm0.7}$ & $1.040$ \\
Amazon-Ratings & $\lambda_{\mathrm{attn}}$ & 0.1 & $42.4_{\scriptscriptstyle\pm0.8}$ & $1.022$ \\
Amazon-Ratings & $\lambda_{\mathrm{attn}}$ & 0.5 & $42.3_{\scriptscriptstyle\pm0.8}$ & $1.009$ \\
Amazon-Ratings & $\lambda_{\mathrm{attn}}$ & 1.0 & $42.1_{\scriptscriptstyle\pm0.6}$ & $1.005$ \\
Amazon-Ratings & $T_w$ & 0 & $43.4_{\scriptscriptstyle\pm0.3}$ & $1.457$ \\
Amazon-Ratings & $T_w$ & 50 & $42.6_{\scriptscriptstyle\pm0.9}$ & $1.304$ \\
Amazon-Ratings & $T_w$ & 100 & $42.4_{\scriptscriptstyle\pm0.8}$ & $1.273$ \\
Amazon-Ratings & $T_w$ & 200 & $41.9_{\scriptscriptstyle\pm0.6}$ & $1.000$ \\
Amazon-Ratings & $\kappa$ & 1.0 & $43.2_{\scriptscriptstyle\pm0.6}$ & $1.461$ \\
Amazon-Ratings & $\kappa$ & 5.0 & $42.4_{\scriptscriptstyle\pm0.8}$ & $1.098$ \\
Amazon-Ratings & $\kappa$ & 10.0 & $42.2_{\scriptscriptstyle\pm0.8}$ & $1.045$ \\
Amazon-Ratings & $\kappa$ & 50.0 & $42.0_{\scriptscriptstyle\pm0.6}$ & $1.006$ \\
\midrule
Cora & $\lambda_{\mathrm{attn}}$ & 0.0 & $81.1_{\scriptscriptstyle\pm0.8}$ & $1.000$ \\
Cora & $\lambda_{\mathrm{attn}}$ & 0.01 & $81.1_{\scriptscriptstyle\pm0.8}$ & $1.000$ \\
Cora & $\lambda_{\mathrm{attn}}$ & 0.1 & $81.1_{\scriptscriptstyle\pm0.8}$ & $1.000$ \\
Cora & $\lambda_{\mathrm{attn}}$ & 0.5 & $81.1_{\scriptscriptstyle\pm0.8}$ & $1.000$ \\
Cora & $\lambda_{\mathrm{attn}}$ & 1.0 & $81.1_{\scriptscriptstyle\pm0.8}$ & $1.000$ \\
Cora & $T_w$ & 0 & $81.1_{\scriptscriptstyle\pm0.8}$ & $1.014$ \\
Cora & $T_w$ & 50 & $81.1_{\scriptscriptstyle\pm0.8}$ & $1.000$ \\
Cora & $T_w$ & 100 & $81.1_{\scriptscriptstyle\pm0.8}$ & $1.000$ \\
Cora & $T_w$ & 200 & $81.1_{\scriptscriptstyle\pm0.8}$ & $1.000$ \\
Cora & $\kappa$ & 1.0 & $81.1_{\scriptscriptstyle\pm0.8}$ & $1.000$ \\
Cora & $\kappa$ & 5.0 & $81.1_{\scriptscriptstyle\pm0.8}$ & $1.000$ \\
Cora & $\kappa$ & 10.0 & $81.1_{\scriptscriptstyle\pm0.8}$ & $1.000$ \\
Cora & $\kappa$ & 50.0 & $81.1_{\scriptscriptstyle\pm0.8}$ & $1.000$ \\
\midrule
Roman-Empire & $\lambda_{\mathrm{attn}}$ & 0.0 & $72.8_{\scriptscriptstyle\pm0.4}$ & $1.807$ \\
Roman-Empire & $\lambda_{\mathrm{attn}}$ & 0.01 & $72.8_{\scriptscriptstyle\pm0.3}$ & $1.754$ \\
Roman-Empire & $\lambda_{\mathrm{attn}}$ & 0.1 & $72.4_{\scriptscriptstyle\pm0.3}$ & $1.393$ \\
Roman-Empire & $\lambda_{\mathrm{attn}}$ & 0.5 & $71.4_{\scriptscriptstyle\pm0.3}$ & $1.128$ \\
Roman-Empire & $\lambda_{\mathrm{attn}}$ & 1.0 & $71.0_{\scriptscriptstyle\pm0.7}$ & $1.070$ \\
Roman-Empire & $T_w$ & 0 & $72.7_{\scriptscriptstyle\pm0.5}$ & $1.828$ \\
Roman-Empire & $T_w$ & 50 & $72.5_{\scriptscriptstyle\pm0.5}$ & $1.716$ \\
Roman-Empire & $T_w$ & 100 & $72.2_{\scriptscriptstyle\pm0.4}$ & $1.574$ \\
Roman-Empire & $T_w$ & 200 & $70.6_{\scriptscriptstyle\pm0.5}$ & $1.000$ \\
Roman-Empire & $\kappa$ & 1.0 & $72.7_{\scriptscriptstyle\pm0.5}$ & $1.809$ \\
Roman-Empire & $\kappa$ & 5.0 & $71.9_{\scriptscriptstyle\pm0.5}$ & $1.348$ \\
Roman-Empire & $\kappa$ & 10.0 & $71.4_{\scriptscriptstyle\pm0.6}$ & $1.193$ \\
Roman-Empire & $\kappa$ & 50.0 & $70.8_{\scriptscriptstyle\pm0.4}$ & $1.041$ \\
\midrule
Wisconsin & $\lambda_{\mathrm{attn}}$ & 0.0 & $47.1_{\scriptscriptstyle\pm5.5}$ & $1.000$ \\
Wisconsin & $\lambda_{\mathrm{attn}}$ & 0.01 & $47.1_{\scriptscriptstyle\pm5.5}$ & $1.000$ \\
Wisconsin & $\lambda_{\mathrm{attn}}$ & 0.1 & $47.1_{\scriptscriptstyle\pm5.5}$ & $1.000$ \\
Wisconsin & $\lambda_{\mathrm{attn}}$ & 0.5 & $47.1_{\scriptscriptstyle\pm5.5}$ & $1.000$ \\
Wisconsin & $\lambda_{\mathrm{attn}}$ & 1.0 & $47.1_{\scriptscriptstyle\pm5.5}$ & $1.000$ \\
Wisconsin & $T_w$ & 0 & $47.1_{\scriptscriptstyle\pm5.5}$ & $1.001$ \\
Wisconsin & $T_w$ & 50 & $47.1_{\scriptscriptstyle\pm5.5}$ & $1.000$ \\
Wisconsin & $T_w$ & 100 & $47.1_{\scriptscriptstyle\pm5.5}$ & $1.000$ \\
Wisconsin & $T_w$ & 200 & $47.1_{\scriptscriptstyle\pm5.5}$ & $1.000$ \\
Wisconsin & $\kappa$ & 1.0 & $47.1_{\scriptscriptstyle\pm5.5}$ & $1.000$ \\
Wisconsin & $\kappa$ & 5.0 & $47.1_{\scriptscriptstyle\pm5.5}$ & $1.000$ \\
Wisconsin & $\kappa$ & 10.0 & $47.1_{\scriptscriptstyle\pm5.5}$ & $1.000$ \\
Wisconsin & $\kappa$ & 50.0 & $47.1_{\scriptscriptstyle\pm5.5}$ & $1.000$ \\
\bottomrule
\end{tabular}
}
\end{table}

\begin{table}[h]
\centering
\caption{Ablation: attention regularisation $\lambda_{\mathrm{attn}}$ on Roman-Empire.}
\label{tab:abl_lambda}
\begin{tabular}{cccc}
\toprule
$\lambda_{\mathrm{attn}}$ & Accuracy (\%) & Mean $q$ & ECE \\
\midrule
0.0 & $72.8_{\scriptscriptstyle\pm0.4}$ & 1.807 & 0.094 \\
0.01 & $\mathbf{72.8}_{\scriptscriptstyle\pm0.3}$ & 1.754 & 0.096 \\
0.1 & $72.4_{\scriptscriptstyle\pm0.3}$ & 1.393 & 0.093 \\
0.5 & $71.4_{\scriptscriptstyle\pm0.3}$ & 1.128 & 0.092 \\
1.0 & $71.0_{\scriptscriptstyle\pm0.7}$ & 1.070 & 0.090 \\
\bottomrule
\end{tabular}

\end{table}

\begin{table}[h]
\centering
\caption{Ablation: warm-up duration $T_w$ on Roman-Empire.}
\label{tab:abl_warmup}
\begin{tabular}{cccc}
\toprule
$T_w$ & Accuracy (\%) & Mean $q$ & ECE \\
\midrule
0 & $\mathbf{72.7}_{\scriptscriptstyle\pm0.5}$ & 1.828 & 0.090 \\
50 & $72.5_{\scriptscriptstyle\pm0.5}$ & 1.716 & 0.095 \\
100 & $72.2_{\scriptscriptstyle\pm0.4}$ & 1.574 & 0.094 \\
200 & $70.6_{\scriptscriptstyle\pm0.5}$ & 1.000 & 0.090 \\
\bottomrule
\end{tabular}

\end{table}

\begin{table}[h]
\centering
\caption{Ablation: learning-rate ratio $\kappa = \eta_\theta/\eta_\alpha$ on Roman-Empire.}
\label{tab:abl_kappa}
\begin{tabular}{cccc}
\toprule
$\kappa$ & Accuracy (\%) & Mean $q$ & ECE \\
\midrule
1.0 & $\mathbf{72.7}_{\scriptscriptstyle\pm0.5}$ & 1.809 & 0.093 \\
5.0 & $71.9_{\scriptscriptstyle\pm0.5}$ & 1.348 & 0.093 \\
10.0 & $71.4_{\scriptscriptstyle\pm0.6}$ & 1.193 & 0.091 \\
50.0 & $70.8_{\scriptscriptstyle\pm0.4}$ & 1.041 & 0.093 \\
\bottomrule
\end{tabular}

\end{table}

\begin{table}[h]
\centering
\caption{Ablation: reparameterisation half-width $\delta$ on Roman-Empire.}
\label{tab:abl_delta}
\begin{tabular}{cccc}
\toprule
$\delta$ & Accuracy (\%) & Mean $q$ & ECE \\
\midrule
0.25 & $71.6_{\scriptscriptstyle\pm0.4}$ & 1.223 & 0.093 \\
0.5 & $72.0_{\scriptscriptstyle\pm0.2}$ & 1.434 & 0.093 \\
1.0 & $72.5_{\scriptscriptstyle\pm0.6}$ & 1.808 & 0.090 \\
1.5 & $\mathbf{72.7}_{\scriptscriptstyle\pm0.5}$ & 2.118 & 0.090 \\
2.0 & $72.2_{\scriptscriptstyle\pm0.6}$ & 2.354 & 0.085 \\
\bottomrule
\end{tabular}

\end{table}

\begin{table}[h]
\centering
\caption{Granularity ablation on Roman-Empire: shared (Global), per-layer, per-head, or per-edge $q$.  ``Mean $q$'' is the converged entropic index for the scalar variants; for LTGA-Edge, which has no scalar $q$, the entry is the gate's output-weight magnitude and is marked with an asterisk so it cannot be read as a learned $q$.}
\label{tab:abl_granularity}
\begin{tabular}{cccc}
\toprule
Granularity & Accuracy (\%) & Mean $q$ & ECE \\
\midrule
LTGA-Edge & $\mathbf{72.9}_{\scriptscriptstyle\pm0.3}$ & 0.377$^{\ast}$ & 0.087 \\
LTGA-Global & $72.6_{\scriptscriptstyle\pm0.5}$ & 1.880 & 0.093 \\
LTGA-Head & $72.5_{\scriptscriptstyle\pm0.2}$ & 1.809 & 0.088 \\
LTGA-Layer & $72.7_{\scriptscriptstyle\pm0.6}$ & 1.853 & 0.092 \\
\bottomrule
\end{tabular}

\end{table}

\section{Depth and over-smoothing}\label{app:depth}

Sparse attention should, in principle, slow over-smoothing: pruned
edges do not mix representations.  Table~\ref{tab:depth} sweeps
$L\!\in\!\{2,4,8\}$ for GCN, GAT, GATv2 and LTGA-Head on a homophilic (Cora)
and a heterophilic (Roman-Empire) graph, and reports two collapse diagnostics
alongside accuracy: the graph Dirichlet energy of the last hidden layer
(normalised so it tracks directional rather than magnitude collapse) and the
mean average distance $1-\overline{\cos}$ over node pairs.  The claim worth
testing is not that LTGA is more accurate when deep but that its
representations collapse more slowly, which requires the diagnostics rather
than accuracy alone.  The outcome is negative: at $L\!=\!8$ every model
collapses, and the diagnostics fall with the accuracy rather than separating
the methods --- on Cora LTGA-Head reaches $22.8\%$ ($E_{\rm dir}\!=\!0.007$,
MAD $0.204$) against $28.5\%$ for GAT ($0.006$, $0.188$), and on Roman-Empire
all four models sit between $11\%$ and $14\%$ with $E_{\rm dir}\!\approx\!0$.
Sparse attention does not mitigate over-smoothing here.  LTGA-Head does retain
a small advantage at $L\!=\!4$ on Roman-Empire ($68.5\!\pm\!2.1$ vs
$67.4\!\pm\!0.7$ for GATv2), so the mechanism is not a two-layer artefact, but
the intervals overlap and this sweep uses five seeds.

\begin{table}[h]
\centering
\caption{Depth sweep with over-smoothing diagnostics.  Higher
Dirichlet energy and higher MAD mean less collapse.}
\label{tab:depth}
\resizebox{\textwidth}{!}{
\begin{tabular}{llccccccccc}
\toprule
Dataset & Model & \multicolumn{3}{c}{$L=2$} & \multicolumn{3}{c}{$L=4$} & \multicolumn{3}{c}{$L=8$} \\
\cmidrule(lr){3-5} \cmidrule(lr){6-8} \cmidrule(lr){9-11}
 & & acc & $E_{\mathrm{dir}}$ & MAD & acc & $E_{\mathrm{dir}}$ & MAD & acc & $E_{\mathrm{dir}}$ & MAD \\
\midrule
Cora & GAT & $80.6_{\scriptscriptstyle\pm0.5}$ & $0.127$ & $0.717$ & $70.8_{\scriptscriptstyle\pm6.7}$ & $0.039$ & $0.636$ & $28.5_{\scriptscriptstyle\pm9.6}$ & $0.006$ & $0.188$ \\
Cora & GATv2 & $81.0_{\scriptscriptstyle\pm0.5}$ & $0.108$ & $0.668$ & $65.3_{\scriptscriptstyle\pm6.1}$ & $0.023$ & $0.422$ & $17.1_{\scriptscriptstyle\pm7.9}$ & $0.003$ & $0.092$ \\
Cora & GCN & $80.9_{\scriptscriptstyle\pm0.5}$ & $0.155$ & $0.790$ & $59.4_{\scriptscriptstyle\pm11.1}$ & $0.034$ & $0.487$ & $44.5_{\scriptscriptstyle\pm18.2}$ & $0.030$ & $0.342$ \\
Cora & LTGA-Head & $81.0_{\scriptscriptstyle\pm0.4}$ & $0.099$ & $0.667$ & $64.0_{\scriptscriptstyle\pm8.3}$ & $0.031$ & $0.525$ & $22.8_{\scriptscriptstyle\pm12.4}$ & $0.007$ & $0.204$ \\
\midrule
Roman-Empire & GAT & $55.6_{\scriptscriptstyle\pm0.5}$ & $0.336$ & $0.679$ & $46.9_{\scriptscriptstyle\pm3.5}$ & $0.253$ & $0.758$ & $14.0_{\scriptscriptstyle\pm0.2}$ & $0.000$ & $0.005$ \\
Roman-Empire & GATv2 & $71.0_{\scriptscriptstyle\pm0.8}$ & $0.602$ & $0.721$ & $67.4_{\scriptscriptstyle\pm0.7}$ & $0.568$ & $0.847$ & $12.0_{\scriptscriptstyle\pm1.7}$ & $0.000$ & $0.003$ \\
Roman-Empire & GCN & $42.9_{\scriptscriptstyle\pm0.4}$ & $0.162$ & $0.434$ & $29.9_{\scriptscriptstyle\pm0.3}$ & $0.086$ & $0.646$ & $12.7_{\scriptscriptstyle\pm5.1}$ & $0.003$ & $0.024$ \\
Roman-Empire & LTGA-Head & $72.4_{\scriptscriptstyle\pm0.6}$ & $0.634$ & $0.752$ & $68.5_{\scriptscriptstyle\pm2.1}$ & $0.580$ & $0.837$ & $11.2_{\scriptscriptstyle\pm4.0}$ & $0.000$ & $0.010$ \\
\bottomrule
\end{tabular}
}
\end{table}


\end{document}